\documentclass[10pt,twocolumn,letterpaper]{article}

\usepackage[pagenumbers]{wacv} % To force page numbers, e.g. for an arXiv version

\usepackage{amsthm}
\usepackage{multirow}

\newif\ifsupp
\supptrue     % <- true: supplementary only.  false: main paper only.
\AtEndPreamble{
    \theoremstyle{plain}
    \newtheorem{theorem}{Theorem}
    \newtheorem{lemma}[theorem]{Lemma}
    \newtheorem{corollary}[theorem]{Corollary}

    \theoremstyle{definition}
    \newtheorem{definition}[theorem]{Definition}

    \theoremstyle{remark}
    \newtheorem{remark}[theorem]{Remark}

    \AtBeginEnvironment{lemma}{\crefalias{theorem}{lemma}}
    \AtBeginEnvironment{corollary}{\crefalias{theorem}{corollary}}
    \AtBeginEnvironment{definition}{\crefalias{theorem}{definition}}
    \AtBeginEnvironment{remark}{\crefalias{theorem}{remark}}

    \crefname{theorem}{Theorem}{Theorems}
    \crefname{lemma}{Lemma}{Lemmas}
    \crefname{corollary}{Corollary}{Corollaries}
    \crefname{definition}{Definition}{Definitions}
    \crefname{remark}{Remark}{Remarks}
    \Crefname{theorem}{Theorem}{Theorems}
    \Crefname{lemma}{Lemma}{Lemmas}
    \Crefname{corollary}{Corollary}{Corollaries}
    \Crefname{definition}{Definition}{Definitions}
    \Crefname{remark}{Remark}{Remarks}
}

\AtEndPreamble{
  \newcommand{\suppref}[2]{\ifcsname r@#1\endcsname\cref{#1}\else #2\fi}
}

\definecolor{wacvblue}{rgb}{0.21,0.49,0.74}
\usepackage[pagebackref,breaklinks,colorlinks,allcolors=wacvblue]{hyperref}

\def\wacvPaperID{1160} % *** Enter the WACV Paper ID here
\def\confName{WACV}
\def\confYear{2027}

\title{Right In-Place (RiP) Convolution: A Simple, General, and Near-Optimal Strategy
for Memory-Efficient CNN Inference}
\author{Opegbemi Matthias Busoye \quad Tolulope Matthew Busoye \quad Eghonghon-aye Eigbe\\
PowerLabs Technologies, Lagos, Nigeria\\
{\tt\small \{matthias, matthew, eghonghon\}@powerlabstech.com}
}

\begin{document}
\maketitle
\begin{abstract}
Activation memory, not compute, limits CNN inference on constrained hardware such as microcontrollers. Direct in-place convolution removes the dual-buffer cost, but the memory-optimal formulation of Gural and Murmann~\cite{pmlr-v97-gural19a} assumes valid padding, unit stride, unit dilation, and odd square kernels, and needs a non-sequential traversal costing $2\times$ inference time in transposes. We identify two regimes in which their published closed form does not hold: (1) an under-allocation of exactly $(k-1)C_{in} \bmod (C_{out}-C_{in})$ scalars, active on every convolutional layer of their own deployed network and manifesting as a silent corruption of still-live input; (2) an unbounded overestimate, up to $2{,}432\times$, once the critical leg leaves the output grid. We correct both and generalize to arbitrary stride, dilation, padding, and rectangular kernels. We then propose Right In-Place (RiP) convolution, a bit-identical operation in which every layer reads its input right-aligned in a shared workspace and writes its output left-aligned from index zero. The debt is piecewise affine in the output pixel index, so evaluating its breakpoints in $O(1)$ yields the minimum safe gap without enumerating the output grid, with row-major access preserved. Across $10{,}000$ random layers RiP produced no corruption, and across 84 convolutional layers from 25 architectures it matches the herringbone workspace exactly on 58 and within 5\% on 81, using 24.8\% less memory than dual buffering on average. Written into TinyEngine's~\cite{lin2020mcunettinydeeplearning} kernels and deployed to a Raspberry Pi Pico 1 and Pico 2, it cuts peak activation memory across eleven MCUNet models by 12.5 to 33.3\% at unchanged cycle counts and bit-identical outputs, raising the number of models that fit the Pico 1's 256\,KB SRAM from six to nine.
% \keywords{In-place convolution \and Memory-efficient inference \and
% Microcontrollers \and Edge AI \and TinyML}
\end{abstract}
    
\section{Introduction}
\label{sec:intro}

Inference is moving from the cloud to the edge, where running locally keeps
user data on the device, tolerates intermittent connectivity, and fits a
single-Joule power envelope. The memory hierarchy is what constrains it: a
modern GPU has tens of gigabytes of DRAM, a mid-range microcontroller (MCU)
such as the RP2040 has 256\,KB\footnote{We report 256\,KB, not the RP2040
full 264\,KB, since the last two 4\,KB banks (SCRATCH\_X/SCRATCH\_Y) are
reserved by the Pico SDK's default linker script for per-core stacks and are
not part of the contiguous RAM region available to the application.} of SRAM,
and the low-end 8-bit class still shipped by the billion provides 2--16\,KB.

That makes deployment a memory problem before it is a compute problem, and one
question decides whether a network runs well: does the working set fit in the
fast memory tier, or must it spill to a slower and more energy-expensive
one~\cite{Horowitz201411CE}? Weights are by now well served, with
pruning, quantization, and bit-packing reaching compression above
$50\times$~\cite{busoye2026dmc}. Activations are not. They must be materialized
during inference, and for the early layers of a typical CNN, the input feature
map alone can exceed the budget. Microcontrollers make this hardest and motivate
our analysis, but nothing in it is specific to them: the same accounting governs
any device whose working set competes for a fast memory tier.

Direct in-place convolution, overwriting stale input activations with new outputs as computation proceeds, needs no auxiliary buffer, no cross-layer scheduling, and no hardware modification. It works because convolution sits between two extremes. In-place execution is routine for an elementwise operator such as ReLU, where each output depends on one input, and impossible for a dense layer $y = Wx$, where nothing goes stale until the whole output exists. Specifically, a convolution output depends on a bounded neighborhood of $k_h k_w C_{in}$ inputs, so what is needed is not a second buffer but a \emph{gap}: a bounded region of slack memory between the write frontier and the next live input. Determining that gap is the problem this paper solves.

Gural and Murmann~\cite{pmlr-v97-gural19a} prove that a \emph{herringbone} traversal minimizes activation memory for direct convolution, fitting an MNIST classifier into 2\,KB of SRAM. Two gaps remain. Their analysis assumes valid padding, unit stride, unit dilation, and odd square kernels. Modern CNNs violate every one of these assumptions: strided downsampling~\cite{liu2022convnet2020s}, dilated convolutions~\cite{chen2018encoderdecoderatrousseparableconvolution}, non-square kernels~\cite{szegedy2015rethinkinginceptionarchitecturecomputer}, and same-padding are staples of ResNet~\cite{residual_7780459}, MobileNet~\cite{howard2017mobilenetsefficientconvolutionalneural}, EfficientNet~\cite{tan2020efficientnetrethinkingmodelscaling}, and MCUNet~\cite{lin2020mcunettinydeeplearning}. Also, the geometric lemma at its heart does not generalize to anisotropic receptive fields. Moreover, optimality demands a non-sequential traversal, forcing an in-place transpose at every row--column switch and a final inverse permutation, which the authors report costs $2\times$ inference time.

\paragraph{Contributions.} We identify two regimes in which the published
closed form fails, one under-allocating and one overestimating without bounds,
and correct both. We generalize the analysis to arbitrary stride, dilation,
padding, rectangular kernels, and uneven output grids. We then propose
\emph{Right In-Place} (RiP) convolution, a row-major layout whose minimum safe
gap is obtained by breakpoint evaluation, and validate it on $10{,}000$
synthetic layers, 84 layers from 25 architectures, and nine and eleven MCUNet models
deployed to a Raspberry Pi Pico 1 and Pico 2 respectively.

\begin{figure*}[!ht]
    \centering
    \includegraphics[width=1\textwidth]{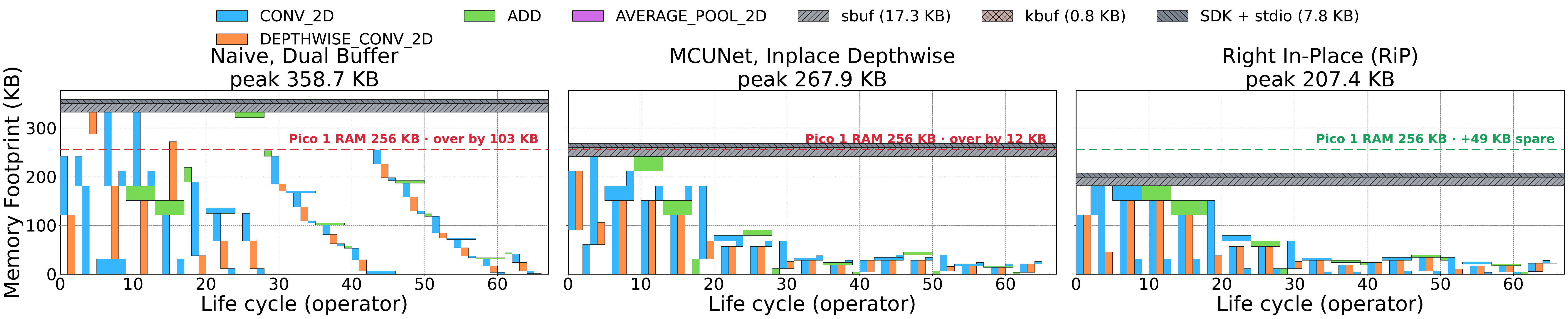}
    \caption{Workspace layout on a Pico 1 for MCUNet-in3 under naive dual buffering, TinyEngine's scheduler, and RiP. All three are placed by the same MILP memory allocator (\cref{sec:execution}). On the first two it reaches the same peak as TinyEngine's first-fit memory allocator~\cite{lin2020mcunettinydeeplearning}, so the allocator alone does not account for the saving; under in-place convolution first-fit fragments the arena and does not allow the full utilization of the extra memory, so the MILP is what lets RiP's gap be reclaimed rather than stranded.}
    \label{fig:mcuin3_memory_layout}
\end{figure*}

\section{Related Work}
\label{sec:related}

Work on activation memory for kilobyte-scale devices divides along the axis of
memory reuse that each method exploits. \cref{tab:related_axes} places our contribution
(a fine-grained \emph{spatial} in-place scheme for general channel-mixing
convolutions) against the rest.

\paragraph{Spatial in-place convolution.}
Gural \& Murmann~\cite{pmlr-v97-gural19a} first showed that direct convolution
can overwrite stale input activations with new outputs, and proved that a
\emph{herringbone} traversal of the output grid, alternating rows and columns
so that edge pixels are reclaimed as early as possible, achieves the minimum
activation memory under valid padding, unit stride, unit dilation, and odd
square kernels. The result is memory-optimal but costs a non-sequential output
order, an in-place transpose at every row-to-column switch, and a final inverse
permutation; their Arduino case study reports $2\times$ inference overhead
($819\,\mathrm{s}$ against $396\,\mathrm{s}$). Their concurrently introduced
\emph{Replace} strategy runs row-major at naive speed, but is only analyzed for 
the same restricted setting and fragments freed memory at row ends. Neither has
been extended to arbitrary stride, dilation, padding, or rectangular kernels,
and few subsequent methods have revisited the layout problem itself. Zhang et
al.~\cite{zhang2018highperformancezeromemoryoverhead} reorder loops for
``zero-memory-overhead'' direct convolution, but measure overhead \emph{beyond}
the storage of the input and output activations, so a full output buffer is
still allocated. RiP extends the Gural \& Murmann analysis to arbitrary
parameters and removes the transpose overhead, using a fragmentation-free
single-gap layout with an $O(1)$ closed-form gap.

The closest confusable work is buffer replacement. Sakr et
al.~\cite{sakr2021memory} extend CMSIS-NN with a single arena sized to the
largest pair of adjacent layer activations so that every layer draws its input
and output from the same reused space, cutting whole-network activation memory
by roughly $27.2\%$. However, within a layer, the convolution is still dual-buffered; the
granularity is the whole tensor, and the decision is made between layers so a
layer's own input and output coexist while it runs. RiP overlaps those two in
one buffer sized by the worst instant, and the two compose: our per-layer bound
sets the size of the very tensors such an arena must pack.

\paragraph{Other axes.}
Channel in-place exploits the fact that depthwise convolutions do not mix
channels, so the input activation of channel $\gamma$ may be overwritten by the
output of channel $\gamma{+}1$. TinyEngine~\cite{lin2020mcunettinydeeplearning,
Lin_2023} measures a $1.6\times$ peak reduction on depthwise layers and
TinyAD~\cite{tinyad_10093068} formalizes the same reschedule with a transit
buffer. Neither can apply to the $3{\times}3$ and pointwise $1{\times}1$
convolutions that dominate modern activation memory, which is exactly the case
RiP targets. Cross-layer methods attack the peak from outside the layer:
MCUNetV2~\cite{lin2024mcunetv2memoryefficientpatchbasedinference} computes a
spatial patch across several layers before advancing, cutting MobileNetV2's
peak by up to $8\times$ at the cost of recomputation, and fused depthwise
tiling~\cite{stahl2023fuseddepthwisetilingmemory} and line-buffer
pipelining~\cite{kang2022aocstreamallonchipcnnaccelerator} reach the same end
by other means. All still pay the full
$H_pW_pC_{in} + H_{out}W_{out}C_{out}$ inside any single bottleneck layer; RiP
attacks that residual, and the patch itself can be executed in our layout. In
silicon, RAMAN~\cite{krishna2023ramanreconfigurablesparsetinyml} overlays the
input and output activation buffers on one on-chip region for up to $50\%$
storage reduction --- the coarse, whole-buffer analogue of our idea, tied to a
specific dataflow and requiring custom hardware.

\textbf{Memory planning.} A parallel line treats activation memory as
tensor-lifetime allocation. TFLite Micro packs tensors into a fixed arena by
first fit~\cite{david2021tensorflowlitemicroembedded}; iterative graph
optimization~\cite{zhong2023memoryawareschedulingcomplexwired} cuts the peak by
up to 26.4\% against it, and MoteNN~\cite{motenn_10.1145/3649329.3655922} by up to 80\% against prior
coarse-grained schedulers. These answer ``given a set of tensor lifetimes, how
tightly can we pack them?''
% A parallel line treats activation memory as tensor-lifetime allocation. TFLite
% Micro First fit allocation~\cite{david2021tensorflowlitemicroembedded}, iterative graph
% optimization~\cite{zhong2023memoryawareschedulingcomplexwired}, and
% MoteNN~\cite{motenn} pack tensors into a fixed budget
% given a computation schedule, reporting up to $26.4\%$ and up to $80\%$ peak
% reduction respectively over topology-based scheduling. 
% These answer
% ``\emph{given} a set of tensor lifetimes, how tightly can we pack them?'' 
RiP
answers the lower-level question of what the smallest workspace is under which
a single convolution can execute at all, so a tighter per-layer bound shrinks
the tensors any planner must pack. Weight
quantization~\cite{han2016deepcompressioncompressingdeep,
nagel2020downadaptiveroundingposttraining,
hubara2020improvingposttrainingneural},
pruning~\cite{han2015learningweightsconnectionsefficient,
luo2017thinetfilterlevelpruning, he2017channelpruningacceleratingdeep},
architecture search~\cite{cai2019proxylessnasdirectneuralarchitecture,
cai2020onceforalltrainnetworkspecialize}, bit-packing~\cite{busoye2026dmc},
and fused-layer accelerators~\cite{Alwani2016FusedlayerCA} all reduce total
model footprint or cross-layer traffic without altering the intra-layer
activation peak, and are complementary throughout.

\begin{table*}[t]
\centering
\small
\caption{Memory-reduction axes for CNN inference on tiny devices. ``General'' = applies to standard channel-mixing convolutions; ``Sequential'' = preserves row-major output order without reordering; ``Bound'' = derives a closed-form minimum workspace for the layer.}
\label{tab:related_axes}
\begin{tabular}{@{}lccccl@{}}
\toprule
\textbf{Method} & \textbf{Axis} & \textbf{General} & \textbf{Sequential} & \textbf{Bound} & \textbf{Restrictions} \\
\midrule
Herringbone~\cite{pmlr-v97-gural19a}      & spatial in-place & \checkmark & $\times$   & \checkmark & valid, $s{=}1$, $d{=}1$, odd $k$ square \\
Replace~\cite{pmlr-v97-gural19a}          & spatial in-place & \checkmark & \checkmark & $\times$   & valid, $s{=}1$, $d{=}1$, odd $k$ square \\
In-place DW~\cite{lin2020mcunettinydeeplearning}     & channel in-place & $\times$   & \checkmark & --         & depthwise only \\
MCUNetV2~\cite{lin2024mcunetv2memoryefficientpatchbasedinference}          & cross-layer patch& \checkmark & --         & --         & recomputation overhead \\
RAMAN~\cite{krishna2023ramanreconfigurablesparsetinyml}            & hardware overlay & \checkmark & --         & $\times$   & requires custom accelerator \\
CMSIS-NN Replace~\cite{sakr2021memory} & inter-layer arena & \checkmark & \checkmark & $\times$   & whole-tensor granularity \\
MoteNN~\cite{motenn_10.1145/3649329.3655922}          & tensor packing   & \checkmark & --         & --         & orthogonal to per-layer bound \\
\textbf{RiP (ours)}       & spatial in-place & \checkmark & \checkmark & \checkmark & --- \\
\bottomrule
\end{tabular}
\end{table*}

\section{Generalized Herringbone Debt}
\label{sec:herringbone}

The herringbone traversal (as in~\cite{pmlr-v97-gural19a}) processes output pixels in an alternating
row--column order (\cref{fig:herringbone_squ}), reclaiming edge input pixels as
they go stale. Under the restricted setting, a leg of length $x$ contributes net
debt $D(x) = x(C_{out}-C_{in}) - (k-1)C_{in}$, which increases with $x$. Legs
are processed in decreasing length, so the running sum rises while legs are long
enough to have a positive debt and falls once they are not; the peak sits where the
sign changes, at the critical point $x^* = (k-1)C_{in}/(C_{out}-C_{in})$. The
as-published closed form~\cite[Eq.~2]{pmlr-v97-gural19a} is

\begin{equation}
\label{eq:gural_eq2}
\begin{split}
D_{\text{hb}}^{\mathrm{G\&M}} = {}&\big(H_{out}W_{out} - \lfloor x^* \rfloor^{2}\big) C_{out} \\
                                  &- \big(H_{in}W_{in} - (\lfloor x^* \rfloor + k - 1)^{2} - k\big) C_{in}.
\end{split}
\end{equation}

\begin{figure}[htbp]
    \centering
    \begin{subfigure}[h]{0.48\linewidth}
        \centering
        \includegraphics[height=3.5cm]{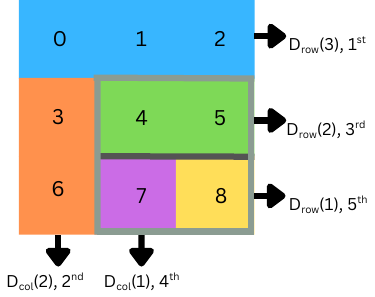}
        \caption{Herringbone order}
        \label{fig:herringbone_squ}
    \end{subfigure}
    \hfill
    \begin{subfigure}[h]{0.48\linewidth}
        \centering
        \includegraphics[height=3.5cm]{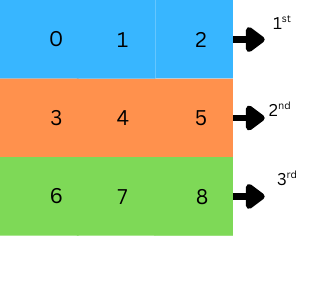}
        \caption{Row-major in-place order}
        \label{fig:in-place-conv-order}
    \end{subfigure}
    \caption{Convolution traversal order.}
    \label{fig:convolution_order}
\end{figure}

\cref{sec:related} noted that the memory-optimal herringbone has never been
extended beyond valid padding, unit stride, unit dilation, and odd square
kernels. In working out that extension we found two independent regimes in which
\cref{eq:gural_eq2} does not hold. This section identifies both limitations
(\cref{sec:limitations}), states a corrected closed form on square output grids
(\cref{sec:corrected}), and summarizes the generalization to arbitrary
convolution parameters (\cref{sec:hb_generalization}).

\subsection{Two Limitations of \cref{eq:gural_eq2}}
\label{sec:limitations}

\paragraph{Limitation 1: the peak is not attained at an integer leg.}
\label{par:limitation1}
\cref{eq:gural_eq2} evaluates the remaining square of unprocessed output at
$\lfloor x^* \rfloor$, which is correct only when $x^*$ is an integer. Take
$k=3$, $C_{in}=3$ (\cref{fig:herringbone_squ} shows the output grid). At $C_{out}=6$ we have
$x^*=2$, the square of side $\lfloor x^* \rfloor$ falls exactly between
processing $D_{col}(2)$ and $D_{row}(2)$, and since $D(x^*)=0$, the formula
evaluates the peak correctly. At $C_{out}=8$ we have $x^*=1.2$, and the peak now sits
between $D_{row}(2)$ and $D_{col}(1)$: the square of side
$\lfloor x^* \rfloor$ is $1\times1$ (only pixel 8), but the unprocessed output at that moment
is the $\lfloor x^* \rfloor \times \lceil x^* \rceil$ rectangle, $1\times2$ (pixels 7 and 8).
The formula therefore evaluates a peak that omits one leg of output.

The shortfall is exactly $(k-1)C_{in} \bmod (C_{out}-C_{in})$, derived in
\suppref{lem:flooring_error}{the supplementary material}\footnote{Proofs, the generalized derivation, the validation protocol and the
per-layer results are in the supplementary material.}. An implementation allocating
$D_{hb}^{\mathrm{G\&M}}$ scalars therefore lets the write frontier advance into
input an unprocessed output still needs, overwriting that many scalars. Two
things keep this silent: the write lands inside the
workspace, so it raises no runtime fault and the count is small (one or two scalars on the layers of
\cref{tab:limitations} against low-bit activations). However, it does not stay small,
being bounded by $C_{out}-C_{in}$: a $16\to64$, $k=3$ layer loses 32. All three
convolutional layers of the network deployed in~\cite{pmlr-v97-gural19a} have
non-integer $x^*$ (\cref{tab:limitations}).

\paragraph{Limitation 2: unbounded overestimate outside the grid.}
\cref{eq:gural_eq2} implicitly assumes the square of side
$\lfloor x^* \rfloor$ fits inside the output grid, that is
$\lfloor x^* \rfloor \le \min(H_{out}, W_{out})-1$. That assumption fails
whenever $C_{out}-C_{in}$ is small relative to $(k-1)C_{in}$: a
near-channel-preserving layer whose kernel is large relative to its output
grid. The quadratic terms are then evaluated outside the grid and return peak
debts larger than the layer's whole activation footprint, by up to
$2{,}432\times$ on the configurations of \cref{tab:limitations}.

\begin{table}[t]
\centering
\caption{The two limitations in the as-published herringbone formula
(\cref{eq:gural_eq2}). \textbf{Truth} is the minimum corruption-free gap, found
by binary search against the simulator; \textbf{Err.} is G\&M minus Truth, so
negative entries are under-allocations. \cref{cor:corrected_hb} reproduces
Truth on every row. $H{=}W$ is the input spatial size.}
\label{tab:limitations}
\scriptsize
\setlength{\tabcolsep}{8pt}
\begin{tabular}{@{}rrrrrrrr@{}}
\toprule
$C_{in}$ & $C_{out}$ & $k$ & $H{=}W$ & $x^*$ & G\&M & Truth & Err. \\
\midrule
\multicolumn{8}{@{}l}{\emph{Limitation 1: G\&M's deployed MNIST network, conv\_0--conv\_2}}\\
1 & 5 & 3 & 14 & 0.5 & 531 & 533 & $-2$ \\
5 & 8 & 3 & 12 & 3.3 & 148 & 149 & $-1$ \\
8 & 11 & 3 & 10 & 5.3 & 45 & 46 & $-1$ \\
\addlinespace
\multicolumn{8}{@{}l}{\emph{Limitation 2: critical leg exceeds the output grid}}\\
3 & 4 & 5 & 8 & 12 & 79 & 7 & $+72$ \\
8 & 9 & 7 & 9 & 48 & 2{,}081 & 11 & $+2{,}070$ \\
16 & 17 & 9 & 12 & 128 & 15{,}520 & 20 & $+15{,}500$ \\
16 & 17 & 15 & 18 & 224 & 48{,}640 & 20 & $+48{,}620$ \\
\bottomrule
\end{tabular}
\end{table}

\subsection{Corrected Closed Form}
\label{sec:corrected}

Limitation~1 is repaired by the partial-leg remainder $(\rho - \alpha_{hb}
\lfloor x^* \rfloor) = (k-1)C_{in} \bmod (C_{out}-C_{in})$, which is the debt sum leaves unpaired at its lower bound; Limitation~2, by clamping the critical leg to the grid it is supposed to
lie in. Both are proved in \suppref{app:corrected_hb_proof}{the supplementary material}.

\begin{corollary}[Corrected restricted form]
\label{cor:corrected_hb}
Under the restrictions of~\cite{pmlr-v97-gural19a}, write
$\alpha_{hb} \equiv C_{out}-C_{in}$, $\rho \equiv (k-1)C_{in}$, and
\begin{equation}
\hat{x} \;=\; \min\!\big(\lfloor x^* \rfloor,\; W_{out}-1\big),
\qquad x^* = \rho/\alpha_{hb}.
\label{eq:correction_limit2}
\end{equation}
If $\alpha_{hb} \le 0$ then $D_{hb} = C_{out}$. Otherwise
\begin{equation}
\begin{split}
D_{hb} &=  C_{out}\big(H_{out}W_{out} - \hat{x}^2\big)  \\
            & - C_{in}(H_{in}W_{in} - (\hat{x}+k-1)^2 - k)\\
            & + (\rho - \alpha_{hb}\hat{x}).
\end{split}
\label{eq:corrected_hb}
\end{equation}
\end{corollary}

It is a corollary in the literal sense: specializing the generalized theorem of
\cref{sec:hb_generalization} to square output, valid padding, unit
stride and unit dilation returns \cref{eq:corrected_hb}, and our implementation has
no separate restricted path. \suppref{app:corrected_hb_proof}{The supplementary material} nonetheless proves
it directly from the restricted accounting, since that is the setting in which
\cref{eq:gural_eq2} was stated and the comparison has to be made there.

\subsection{Generalization and the Restricted Case}
\label{sec:hb_generalization}

Lifting the restrictions of~\cite{pmlr-v97-gural19a} breaks the row--column
symmetry. The traversal starts on the longer output axis and alternates, each
major leg pairing with a shorter minor leg, and closes on a final unpaired
major leg. \suppref{app:generalized_hb}{The supplementary material} develops this in full: it defines the
traversal, proves it visits every output pixel exactly once, and derives the
three leg types, whose debts are linear in position within a leg and quadratic
in leg length across legs. The peak is therefore a maximum over a constant
number of candidates and is evaluated in $O(1)$.

We use the generalized form only as a baseline. It gives the exact peak debt of
the generalized \emph{traversal}, not optimality among all traversals: the
isoperimetric argument of~\cite[Lemma~3.1]{pmlr-v97-gural19a} assumes an
isotropic square kernel and does not extend to anisotropic receptive fields.
What it does establish is that the comparison in \cref{sec:real_layers} exists
at all, since 43 of the 84 layers we measure fall outside the parameter range
in which \cref{eq:gural_eq2} is defined. A simulator that executes the
traversal in an allocation-tracked workspace confirms both forms: the
generalized form on all $10{,}000$ synthetic configurations of
\cref{sec:synthetic} and on all 84 real layers, and \cref{cor:corrected_hb} on
every row of \cref{tab:limitations}.

\section{Right In-Place Convolution}
\label{sec:right_align}
 
\cref{sec:herringbone} corrected the herringbone peak debt but kept its non-sequential traversal, and with it the transpose overhead. This section presents Right In-Place (RiP) convolution: a row-major\footnote{Throughout \cref{sec:right_align}, ``row-major`` refers to traversal order in the traditional raster sense, distinct from the ``major/minor axis`` of \cref{sec:herringbone}, which denotes the longer and shorter output dimension.} (\cref{fig:in-place-conv-order}) single-gap layout whose minimum safe gap follows from a piecewise affine debt function. The gap is solved in $O(1)$ by evaluating that function at nine breakpoints, no transpose and no reorder are required, and the derivation places no restriction on stride, dilation, padding, or kernel shape.

\subsection{Workspace Layout and Safety}
\label{sec:right_align_layout}

We allocate one contiguous workspace:

\begin{equation}
\underbrace{[\,0,\ldots,0\,]}_{\Delta\text{ slack memory}}
\;\Big|\;
\underbrace{X[0,0,0],\;\ldots,\;X[H_p{-}1,\,W_p{-}1,\,C_{in}{-}1]}_{\text{padded input, right-aligned}}
\label{eq:layout}
\end{equation}

holding $\Delta$ scalars of slack memory followed by the padded input. Input pixels are stored \emph{right} aligned in the workspace, offset by $\Delta$ and output pixels are written left-aligned from index~$0$ in row-major order, so the input pixel $(r,c)$ sits at $\Delta + (r W_p + c)C_{in}$ and the output pixel $(r,c)$ at $(r W_{out} + c)C_{out}$. 

Safety requires every output write to end before the \emph{next live input (NLI)} begins in the workspace, that being the lowest workspace index still required by an unprocessed output:

\begin{equation}
\mathrm{NLI}(r,c) =
\begin{cases}
\begin{aligned}[t]
&\Delta + \big(r s_h W_p + (c{+}1)s_w\big) C_{in} \\
&\quad \text{if } (c{+}1) < W_{out},
\end{aligned} \\[8pt]
\begin{aligned}[t]
&\Delta + (r{+}1)s_h W_p\, C_{in} \\
&\quad \text{if } (c{+}1) \ge W_{out} \text{ and } (r{+}1) < H_{out},
\end{aligned} \\[8pt]
\begin{aligned}[t]
& + \infty \\
&\quad \text{otherwise (last pixel).}
\end{aligned}
\end{cases}
\label{eq:nli}
\end{equation}

The stride-aware factor $(c{+}1)s_w$ rather than $c{+}1$ is the key correction: after output pixel $c$ is written with $s_w > 1$, input columns $cs_w$ through $(c{+}1)s_w - 1$ are already stale, so RiP frees memory earlier than the unit-stride analysis of~\cite{pmlr-v97-gural19a} permits. The second branch applies the same accounting to a row end, where the wider stale region is a full input row; the third marks the last output pixel, after which no input is live.
 
The kernel may read its taps in either of two orders, and the choice changes what the write frontier must clear. The receptive field of an output pixel may be gathered into a temporary buffer first, after which the output no longer depends on the workspace copy; or each tap may be read from its workspace address directly into a register, which requires the input to survive intact until the pixel's last channel is written~\cite{pmlr-v97-gural19a}. Write

\begin{equation}
\sigma =
\begin{cases}
0        & \text{receptive field gathered first,} \\
C_{out}  & \text{taps read from the workspace,}
\end{cases}
\label{eq:sigma}
\end{equation}
 
for the headroom the second order demands. It is not geometric debt: it is what production order costs, since an output pixel must reach its final slot before the inputs it overwrites can be declared stale.

\begin{definition}[Total debt]
\label{def:total_debt}
The total debt after writing output pixel $(r,c)$, before releasing the next stale inputs, is

\begin{equation}
\mathcal{D}(r,c) \;=\;
\mathrm{out\_end}(r,c) - \big(\mathrm{NLI}(r,c) - \Delta\big) + \sigma,
\label{eq:delta_rc}
\end{equation}

where $\mathrm{out\_end}(r,c) = \big(r W_{out} + (c{+}1)\big)C_{out}$ is the first address past the pixel just written. A gap of $\Delta$ scalars is safe at $(r,c)$ exactly when $\Delta \ge \mathcal{D}(r,c)$.
\end{definition}
 
\begin{lemma}[Piecewise affine debt]
\label{lem:piecewise_affine_debt}
Under the layout of \cref{eq:layout} and row-major traversal, $\mathcal{D}$ is affine in $(r,c)$ on each branch of \cref{eq:nli}:
 
\begin{equation}
\label{eq:debt}
\mathcal{D}(r,c) =
\begin{cases}
\begin{aligned}[t]
&\beta\,r + \alpha\,(c{+}1) + \sigma \\
&\quad \text{if } (c{+}1) < W_{\mathrm{out}},
\end{aligned} \\[8pt]
\begin{aligned}[t]
&\beta\,(r{+}1) + \sigma \\
&\quad \text{if } (c{+}1) \ge W_{\mathrm{out}} \text{ and } (r{+}1) < H_{\mathrm{out}},
\end{aligned} \\[8pt]
\begin{aligned}[t]
&\sigma \\
&\quad \text{otherwise,}
\end{aligned}
\end{cases}
\end{equation}
 
where
$\beta \equiv C_{\mathrm{out}} W_{\mathrm{out}} - s_h W_p C_{\mathrm{in}}$ is
the net cost of advancing one output row and
$\alpha \equiv C_{\mathrm{out}} - s_w C_{\mathrm{in}}$ that of advancing one
column. Either may be negative, meaning the step frees more input than the
output it writes consumes.
\end{lemma}
 
\begin{proof}
Substitute \cref{eq:nli} into \cref{eq:delta_rc}. On the first branch,
$\big(r W_{out} + (c{+}1)\big)C_{out} - \big(r s_h W_p + (c{+}1)s_w\big)C_{in}
+ \sigma$ collects to $\beta r + \alpha(c{+}1) + \sigma$. On the second,
$c{+}1 = W_{out}$, so $\mathrm{out\_end} = (r{+}1)W_{out}C_{out}$ and the
subtracted term is $(r{+}1)s_hW_pC_{in}$, giving $\beta(r{+}1) + \sigma$. On
the third nothing is live and only $\sigma$ remains.\qed
\end{proof}

\subsection{The Minimum Gap}
\label{sec:right_align_gap}
 
\begin{theorem}[Minimum sequential gap]
\label{thm:gap}
Let
\begin{equation}
\begin{split}
R_s = \{0,\; H_{out}{-}2,\; H_{out}{-}1\},
C_s = \{0,\; W_{out}{-}2,\; W_{out}{-}1\},
\qquad
\label{eq:breakpoints}
\end{split}
\end{equation}

clipped to non-negative values. The minimum gap admitting no corruption is

\begin{equation}
\Delta^{*} = \max\Big(
  \max_{(r,c)\,\in\, R_s \times C_s} \mathcal{D}(r,c),\;\;
  N_{out} - H_pW_pC_{in}
\Big),
\label{eq:delta_star}
\end{equation}
with $N_{out} = H_{out}W_{out}C_{out}$, and the total workspace is
$\mathcal{W}^{*} = \Delta^{*} + H_pW_pC_{in}$. The result holds for any kernel
shape $k_h \times k_w$, strides $(s_h,s_w) \ge 1$, dilations
$(d_h,d_w) \ge 1$, padding $(p_h,p_w) \ge 0$, and channel ratio
$C_{out}/C_{in} > 0$.
\end{theorem}
 
\begin{proof}
\emph{Nine breakpoints suffice.} By \cref{lem:piecewise_affine_debt},
$\mathcal{D}$ is affine on each of the three branches of \cref{eq:nli}, and the
branches partition the output grid along $c = W_{out}{-}1$ and
$r = H_{out}{-}1$. An affine function on a closed interval attains its maximum
at an endpoint, and the two axes enter \cref{eq:debt} independently, so the
candidate set is the product $R_s \times C_s$: the $-1$ entries are the last
row and column, where the next live input jumps to the following output row,
and the $0$ and $-2$ entries are the extremes of the monotone debt on each
axis either side of that jump. All nine are load-bearing; dropping any one
lowers the computed maximum on some configuration.
 
\emph{Safety.} Take $\Delta = \Delta^{*}$. At any $(r,c)$ the write ends at
$\mathrm{out\_end}(r,c)-1$ and the first live input begins at $\mathrm{NLI}(r,c)$,
separated by $\Delta^{*} - \mathcal{D}(r,c) \ge 0$, which is non-negative by
\cref{eq:delta_star}. No output write reaches a live input, so every output is
computed from uncorrupted inputs.
 
\emph{Tightness.} At the maximizing breakpoint the separation is exactly zero,
so any smaller $\Delta$ places a write inside at least one live input.
 
\emph{Workspace.} The output occupies $[0, N_{out})$ and the input
$[\Delta^{*},\, \Delta^{*} + H_pW_pC_{in})$, so the workspace must hold
$\max\big(N_{out},\, \Delta^{*} + H_pW_pC_{in}\big)$. The second term of
\cref{eq:delta_star} forces $\Delta^{*} + H_pW_pC_{in} \ge N_{out}$, so the
maximum is always the second argument and
$\mathcal{W}^{*} = \Delta^{*} + H_pW_pC_{in}$.
 
\emph{Generality.} \cref{eq:debt} depends only on the flat-index arithmetic of
HWC layout and row-major traversal, neither of which restricts the convolution
parameters. Kernel shape and dilation enter only through $k_h^{\mathrm{eff}}$
and $k_w^{\mathrm{eff}}$ in the output dimensions, stride directly through
$s_h, s_w$, and padding through $H_p, W_p$.\qed
\end{proof}
 
Two corollaries follow by inspection and are worth naming, because between them
they cover most of the layers in a modern network. Since
$(H_{out}{-}1, W_{out}{-}1) \in R_s \times C_s$, we always have
$\Delta^{*} \ge \sigma$. When $\alpha \le 0$ and $\beta \le 0$ the bound is met
with equality: the whole output tensor is absorbed into the space its input
vacates, and only the production-order headroom remains. At unit stride this is
just $C_{out} \le C_{in}$, the same-or-decreasing-channel case, where RiP
degenerates to plain in-place replacement.

\subsection{Execution and Re-Alignment}
\label{sec:execution}
 
Given a workspace of $\Delta^{*} + H_pW_pC_{in}$ scalars with the padded input
right-aligned at offset $\Delta^{*}$, the layer visits $(r,c)$ in row-major
order and writes each output pixel to $(rW_{out}+c)C_{out}$. How that pixel is
computed (nested loops, im2col plus GEMV, or a fused kernel) is
unconstrained. The one prerequisite is that the input is not needed again,
which holds for any layer whose input is not consumed by a skip connection off
the same tensor.

Both tap-read orders appear in this paper, and which one is used is dictated by what is being compared. The deployment of \cref{sec:deployment} relies on TinyEngine's default behavior of gathering the receptive field into a temporary buffer~\cite{lin2020mcunettinydeeplearning}. Because this buffer is an engine-level constant required identically by the dual-buffer baseline, it sits outside the activation arena, setting $\sigma=0$ so the RiP gap carries no production-order headroom penalty. On the other hand, the layer-level comparison of \cref{sec:real_layers} uses $\sigma = C_{out}$, because the herringbone accounting of~\cite{pmlr-v97-gural19a} assumes taps are read from the workspace and matching it is what makes that comparison like for like. RiP is therefore measured against herringbone under herringbone's own execution order, and deployed under the cheaper one.
 
In RiP, output emerges left-aligned while the next layer expects its input
right-aligned, so the two must reconcile. Under dynamic allocation this is
free: the next layer's input pointer is set $\Delta^{*}$ scalars past the buffer
origin and no data moves. A device working from a single predetermined
workspace has no pointer to move, and the obvious remedy is one
$O(H_pW_pC_{in})$ memmove (right shift) per layer against the
$O(H_{out}W_{out}k_hk_wC_{in}C_{out})$ convolution. In practice it does not
come to that, because what a layer needs is slack memory to the left of its input, and
the gap may sit at either end of that layer's allocation. Therefore, a move is forced only
when the layer's input already occupies its workspace from the left.

Consequently, we solve the placement once for the whole network, offline, with a
mixed-integer linear program (MILP) that replaces the greedy first-fit allocator of
TinyEngine's code generator. Each activation tensor $t$ carries a byte offset
variable $x_t$; each candidate in-place input tensor $i$ carries two binaries,
$d_i$ for whether the op overwrites it and $r_i$ for whether it must be
separated from the output by a physical shift ($r_i = 1$) or by placing the
output $\Delta^{*}_i$ below the input ($r_i = 0$). Tensors that are alive together may not share an address, each output overwrites
at most one input, and $r_i \le d_i$. Every offset is constrained to a multiple of $V$ bytes ($V = 4$ on both Picos)
so that aligned SIMD loads stay legal, which also requires the gap $\Delta^*_i$
to be a multiple of $V$.
The objective minimizes the peak address used by the model,
$\max_t(x_t + \mathrm{size}_t)$, plus $C\sum_i \mathrm{size}_i r_i$, so a move
is chosen only where placing the output low would cost more peak than the move
is worth. At $C = 10^{-4}$, a move of the largest tensor in our suite prices at
roughly 30\;B of peak.
It
still returns $r_i = 0$ everywhere: across all eleven deployed models, between 44\%
and 57\% of tensors are placed in-place and not one memmove is scheduled. The re-alignment cost we measure is zero, and the bias in the objective is the reason that is a result.
% than a convenience.

The allocator is not itself a source of saving. On naive dual buffering and on
TinyEngine's own schedule, the program reaches the same peak as first-fit, so
the reductions in \cref{sec:deployment} are attributable to the layout alone.
Under RiP it is load-bearing for a different reason. A layer's freed gap is
useful only if a later tensor can be placed in it, and first-fit commits
offsets in schedule order without a lookahead; on MCUNet-in2 and in3 (\cref{fig:mcuin3_memory_layout}), whose
branch connections keep two tensors alive across the shrunken layer, it strands
the reclaimed space between live tensors and returns no peak reduction at all despite the smaller buffers. Solving placement globally is what converts a per-layer bound into a \emph{network-level} saving. Either way the layout resolves the ``inconveniently laid-out stale memory''
obstacle that led Gural and Murmann to abandon the row-major
setting~\cite[\S3.4.1]{pmlr-v97-gural19a}, without herringbone's $2\times$
transpose overhead. Tensor packing is NP-hard, but the instances are
small: across 33 convolutional \texttt{keras.applications} models the full offline step (parse, model build, and solve) takes under 150\;s; \suppref{app:milp_scale}{the supplementary material} details 
this experiment.
% and results.
% agianst duaal buffer.

% To confirm MILP scalability, we simulated allocation times across all convolutional tf.keras.applications models;
% \suppref{app:milp_scale}{the supplementry material} details these comprehensive graph-level memory
% profiles and our extended in-place formulations for non-convolutional operators.

% \paragraph{Padding need not be materialized.}
% \cref{eq:layout} stores the padded input in full, as
% TensorFlow does~\cite{abadi2016tensorflow}, although the
% $(H_pW_p - H_{in}W_{in})C_{in}$ border cells all hold the same pad value and carry
% no information: for a same-padded $3{\times}3$ convolution on a $32{\times}32$
% feature map they are ${\sim}11\%$ of the input footprint. Storing only the
% $H_{in}W_{in}C_{in}$ real scalars is possible, and \cref{app:dense} carries the
% analysis. \cref{lem:dense_affine_debt} shows the debt stays affine between the
% output indices at which a tap crosses the padding boundary, and
% \cref{thm:gap_dense} evaluates it at those breakpoints, in $O(k_hk_w)$ rather
% than $O(1)$. The result belongs here because its sign inverts. The \emph{gap}
% grows, since a smaller input pre-allocates less slack for the output to
% consume; the \emph{workspace} shrinks, and the workspace is what has to fit in
% SRAM (\cref{eq:dense_saving}). A practitioner comparing the two layouts by gap
% will therefore reject the one that wins. The saving costs only index
% arithmetic, with no recomputation and no change to the output;
% \cref{app:dense_gcd} gives the one regime in which the bound is safe but not
% tight.

\paragraph{Padding need not be materialized.}
\cref{eq:layout} stores the padded input in full, as TensorFlow
does~\cite{abadi2016tensorflow}, although the $(H_pW_p - H_{in}W_{in})C_{in}$
border cells all hold the same pad value: for a same-padded $3{\times}3$
convolution on a $32{\times}32$ feature map they are ${\sim}11\%$ of the input
footprint. Storing only the $H_{in}W_{in}C_{in}$ real scalars is possible, and
\suppref{lem:dense_affine_debt,thm:gap_dense}{the supplementary material} re-derives the bound under this dense
layout, at $O(k_hk_w)$ rather than $O(1)$ because the debt now kinks wherever a
tap crosses the padding boundary. The result belongs here because its sign inverts: the gap grows, since a smaller
input footprint pre-allocates less slack for the output to consume, while the
workspace shrinks. What has to fit in SRAM is the workspace, and
\suppref{eq:dense_saving}{the supplementary material} makes the trade precise. A practitioner
comparing the two layouts by gap will therefore reject the one that wins. The
saving costs only index arithmetic; \suppref{app:dense_gcd}{the supplementary material} gives the one regime
in which the bound is safe but not tight.

\section{Results}
\label{sec:results}

\subsection{Experiment Setup}
\label{sec:setup}
We evaluate three claims separately: that the closed forms are safe and tight (\cref{sec:synthetic}), that they are competitive on the layers from which real networks are built (\cref{sec:real_layers}), and that the savings survive deployment (\cref{sec:deployment}). The first is analytic and is checked against a simulator that never evaluates a closed form; the other two are measured on device.

\subsection{Synthetic Layers}
\label{sec:synthetic}
A traversal is correct exactly when no output write lands on an input scalar that an unwritten output still needs; the dot product itself is unchanged, so liveness is the only property at issue. We generated 10{,}000 layers across 20 parameter categories, spanning the ranges real networks occupy together with small-grid, thin-axis, and degenerate regimes that are rare or absent in practice, and ran each through the corruption simulator under five execution configurations. Of the 10{,}000, 22.6\% fall inside the restrictions of~\cite{pmlr-v97-gural19a}; the rest are reachable only through the generalization in \suppref{thm:generalized_hb}{the supplementary material}.

No configuration produced an unsafe gap. The RiP and generalized herringbone
bounds are not merely safe but tight, equal to the simulator's minimum on
every one of the 10,000 layers.
The dense bound is tight on $99.1\%$ when taps are read directly from the
workspace. The exception is layers in which some real input is never read by
any output and is dead before the layer starts, so that the count of dead
scalars no longer determines which addresses are dead;
\suppref{app:dense_gcd}{the supplementary material} gives the mechanism and the two conditions under which it
costs anything. The bound remains safe there, overshooting by $91.85$ scalars on
average. Gathering the receptive field into a temporary buffer first removes
the exception and the bound is tight on all 10,000.

\subsection{Real Architectures}
\label{sec:real_layers}

We then compared the naive dual-buffer workspace, the generalized herringbone
workspace and the RiP workspace analytically across 84 convolutional layers
drawn from 25 architectures spanning classification, detection and speech;
\suppref{tab:sota_layer_results}{the supplementary material} lists them by family. Layers were selected to prioritize peak activation memory rather than to sample
architectures uniformly: for each network, we take the layers whose combined
input and output footprint is largest, since it is those layers that determine
whether the network fits a given device at all.

Each traversal is first confirmed against the corruption simulator, which executes it
in an allocation-tracked buffer; the generalized herringbone form matches on all
84. We also run the 48 layers that fit on the Pico~2 and show that the right-shift is close to
free: every one costs within $1\%$ of the dual-buffer kernel, five run
marginally faster, and the right-shift is a median $0.05\%$ of layer conv runtime, worst
$0.83\%$. Forty-three of the 84 layers fall outside the parameter range
of~\cite{pmlr-v97-gural19a} (strided, dilated, padded, or rectangular), so the
generalization is not a formality but the condition under which the comparison
exists at all.

RiP matches the herringbone workspace exactly on 58 of the 84 layers and to
within 5\% on 81, at a mean reduction of 24.8\% against naive allocation
(median 25.0\%, best 50.0\%). 
Sequential access therefore costs almost nothing in workspace terms, and where it does, the cost is structural rather than diffuse.
On every layer whose output tensor is no larger than its padded input
(25 of the 84) RiP and herringbone agree exactly, without exception. When
the output is larger, the excess cannot be absorbed into the space the input
vacates and must be carried in the gap, where a row-major frontier accumulates
debt across a full row before releasing any of it. All three layers exceeding
5\% are of this kind: two channel-doubling layers of pruned VGG-16 on small grids, and
the worst at 17.0\%, an Inception-v3 $7\times1$ branch whose kernel is
maximally anisotropic. Since both the row and column bounds are computable in $O(1)$, a
practitioner meeting such a layer can evaluate them and transpose the input and
kernels to process in column order instead, which in this case matches herringbone gap exactly.

\begin{table*}[!ht]
\centering
\caption{Eleven models from the MCUNet suite, baseline versus RiP. \textbf{Buffer} is the scratch and im2col
allocation, which sits outside the arena and is unchanged by RiP. Peak
activation memory is the arena alone, excluding that buffer, the SDK and
TinyUSB \texttt{.data}/\texttt{.bss}, and the stack; it is a property of the
allocation, so it is reported once for both boards. Cost is cycles count on the
Pico~2 and wall-clock time on the Pico~1, whose RP2040 exposes no cycle
counter. A dash marks a model the board could not run, since the arena is not
the only claim on memory. \textbf{Acc.} is int8 top-1 as reported
by~\cite{lin2020mcunettinydeeplearning}, identical under both strategies;
$\dagger$ marks Visual Wake Words rather than ImageNet, and $\ddagger$ a
single-class detector, for which no accuracy is reported. \textbf{Bit-exact} reports agreement of on-device output checksums against baseline.
}
\label{tab:deployment}
\small
\setlength{\tabcolsep}{4pt}
\begin{tabular}{@{}lrrrrrrrrrrrc@{}}
\toprule
& &\multicolumn{3}{c}{Peak activation memory (B)} & \multicolumn{3}{c}{Pico 2 cycles (M)}
& \multicolumn{3}{c}{Pico 1  time (ms)} & & \\
\cmidrule(lr){3-5}\cmidrule(lr){6-8}\cmidrule(lr){9-11}
Model & Buffer (B) & Base & RiP & Ratio & Base & RiP & Ratio & Base & RiP & Ratio & Acc. & Bit-exact \\
\midrule
MCUNet-in0     & 3{,}984 &  46{,}080 & \textbf{ 36{,}864} & ${4}{/5}$ &  39.5 &  39.7 & $1.003$ &  1{,}772.4 &  1{,}772.0 & $1.000$ & $40.4\%$ & \checkmark \\
MCUNet-in1     & 5{,}864 &  92{,}160 & \textbf{ 73{,}728} & ${4}/{5}$ &  70.0 &  70.3 & $1.004$ &  3{,}601.1 &  3{,}607.8 & $1.002$ & $49.9\%$ & \checkmark \\
MCUNet-in2     & 14{,}976 & 204{,}800 & \textbf{179{,}200} & ${7}/{8}$ & 307.0 & 307.2 & $1.001$ & 15{,}978.1 & 15{,}999.5 & $1.001$ & $60.3\%$ & \checkmark \\
MCUNet-in3     & 18{,}536 & 247{,}808 & \textbf{185{,}856} & ${3}/{4}$ & 380.5 & 380.8 & $1.001$ &      -- & 19{,}388.2 &      -- & $61.8\%$ & \checkmark \\
MCUNet-in4    & 16{,}520 & 409{,}600 & \textbf{307{,}200} & ${3}/{4}$ & 708.2 & 708.7 & $1.001$ &      -- &      -- &      -- & $68.0\%$ & \checkmark \\
MbV2-w0.35    & 11{,}816 & 290{,}304 & \textbf{248{,}832} & ${6}/{7}$ & 120.4 & 120.4 & $1.000$ &      -- &      -- &      -- & $49.0\%$ & \checkmark \\
Proxyless-w0.3 & 17{,}792 & 247{,}808 & \textbf{185{,}856} & ${3}/{4}$ & 184.2 & 184.3 & $1.000$ &      -- &  9{,}801.9 &      -- & $56.2\%$ & \checkmark \\
MCUNet-VWW0    & 3{,}456&  57{,}344 & \textbf{ 49{,}152} & ${6}/{7}$ &  31.7 &  31.6 & $0.998$ &  1{,}657.8 &  1{,}664.1 & $1.004$ & $87.3\%^{\dagger}$ & \checkmark \\
MCUNet-VWW1   & 4{,}392 &  89{,}600 & \textbf{ 76{,}800} & ${6}/{7}$ &  63.1 &  62.9 & $0.998$ &  3{,}218.8 &  3{,}216.6 & $0.999$ & $88.9\%^{\dagger}$ & \checkmark \\
MCUNet-VWW2  & 11{,}816  & 165{,}888 & \textbf{145{,}152} & ${7}/{8}$ & 269.8 & 269.5 & $0.999$ & 14{,}222.7 & 14{,}245.1 & $1.002$ & $91.8\%^{\dagger}$ & \checkmark \\
Person-Det  & 11{,}688  & 245{,}760 & \textbf{163{,}840} & ${2}/{3}$ & 70.1 & 70.1 & $1.000$ & -- & 5{,}215.4 & -- & --$^{\ddagger}$ & \checkmark \\
\bottomrule
\end{tabular}
\end{table*}

\subsection{Deployment}
\label{sec:deployment}

We deploy all eleven open-source MCUNet models~\cite{lin2020mcunettinydeeplearning}
to two boards, a Raspberry Pi Pico 2 (RP2350, 520\,KB SRAM) and a Raspberry Pi
Pico 1 (RP2040, 256\,KB SRAM), at each model's native
resolution with int8 weights and activations. The baseline is
TinyEngine's own deployment with in-place depthwise convolution enabled, so the
workspace reduction reported here is measured against the state of the art rather than
against naive dual buffering. RiP changes where activation sits, not how it is computed, so our implementation transparently inherits all their engine-level SIMD vectorization and loop unrolling without modification. Therefore, the comparison is an ablation of the allocation strategy alone. We report the total arena rather than the per-layer gap, which excludes the input footprint, so it overstates the difference between them.
% All workspace comparisons are total scalars rather than gap sizes; the gap alone
% overstates the difference between methods because it excludes the input

% footprint common to both.

RiP reduces peak activation memory on every model, by 12.5 to 33.3\% (mean
19.7\%, 444{,}672\,B in total across the eleven). The reductions land on exact
fractions (\cref{tab:deployment}), and this is the
bound showing through rather than a coincidence.
% ,  with $\max(C_{out}, C_{in})/(C_{out}+C_{in})$ of the peak layer. 
% Each model's peak falls on a
% channel-reducing convolution, where 
Cost is unchanged: the mean cycle ratio on the Pico 2 is 1.0005, and
four of the eleven models run marginally fewer cycles under RiP than under the
baseline. The kernels are identical and RiP adds no work, so the
residual spread is measurement noise. Flash is identical under both strategies, because RiP alters
the placement of activations and not the weights. Energy is not reported separately: both strategies execute identical
instructions against the same on-chip SRAM, so cycle count is a faithful proxy. Accuracy is unchanged because RiP is mathematically exact, and we verify this
directly rather than by inference: the checksum of the output logits, computed
on device, matches the baseline bit for bit on every model.

The saving matters most where the budget binds. \cref{fig:mcuin3_memory_layout} traces
activation memory across MCUNet-in3's operators under all three strategies,
with the fixed runtime overhead drawn as a band above the arena: 18{,}536\,B of
scratch, 7{,}964\,B of SDK and TinyUSB \texttt{.data} and \texttt{.bss}, and
2{,}048\,B of stack, 28{,}548\,B in all. Depthwise-only in-place peaks at
247{,}808\,B, which with the band overruns the Pico 1 by
14{,}212\,B; RiP peaks at 185{,}856\,B and clears it by
\textbf{47{,}740}\,B. Proxyless-w0.3 clears by a similar margin, Person-Det by \textbf{76{,}604}\,B.
None of the three fits without RiP, taking the board from six of the eleven
models to nine.

\section{Conclusion}
\label{sec:conclusion}

We corrected two limitations in the published herringbone closed form: a silent under-allocation active on every convolutional layer of the authors' deployed network, and an unbounded overestimate outside the output grid. We also generalized the analysis to arbitrary stride, dilation, padding, and rectangular kernels. Right In-Place convolution then recovers row-major order at a gap solved by breakpoint evaluation, removing 12.5 to 33.3\% of peak activation memory across eleven deployed models at unchanged cycle count and bit identical output. Because the memory saving costs no time, the choice between the two is not a trade-off. Whether a memory-optimal traversal exists at all for anisotropic receptive fields stays open: the isoperimetric argument for the isotropic case does not extend to them.

\clearpage

{
    \small
    \bibliographystyle{ieeenat_fullname}
    \bibliography{main}

@String(CVPR= {IEEE Conf. Comput. Vis. Pattern Recog.})

@String(CVPR  = {CVPR})

@InProceedings{pmlr-v97-gural19a,
  title = 	 {Memory-Optimal Direct Convolutions for Maximizing Classification Accuracy in Embedded Applications},
  author =       {Gural, Albert and Murmann, Boris},
  booktitle = 	 {Proceedings of the 36th International Conference on Machine Learning},
  pages = 	 {2515--2524},
  year = 	 {2019},
  editor = 	 {Chaudhuri, Kamalika and Salakhutdinov, Ruslan},
  volume = 	 {97},
  series = 	 {Proceedings of Machine Learning Research},
  month = 	 {09--15 Jun},
  publisher =    {PMLR},
  url = 	 {https://proceedings.mlr.press/v97/gural19a.html},
}

@misc{zhang2018highperformancezeromemoryoverhead,
      title={High Performance Zero-Memory Overhead Direct Convolutions}, 
      author={Jiyuan Zhang and Franz Franchetti and Tze Meng Low},
      year={2018},
      eprint={1809.10170},
      archivePrefix={arXiv},
      primaryClass={cs.LG},
      url={https://arxiv.org/abs/1809.10170}, 
}

@misc{chen2018encoderdecoderatrousseparableconvolution,
      title={Encoder-Decoder with Atrous Separable Convolution for Semantic Image Segmentation}, 
      author={Liang-Chieh Chen and Yukun Zhu and George Papandreou and Florian Schroff and Hartwig Adam},
      year={2018},
      eprint={1802.02611},
      archivePrefix={arXiv},
      primaryClass={cs.CV},
      url={https://arxiv.org/abs/1802.02611}, 
}

@INPROCEEDINGS{residual_7780459,
  author={He, Kaiming and Zhang, Xiangyu and Ren, Shaoqing and Sun, Jian},
  booktitle={2016 IEEE Conference on Computer Vision and Pattern Recognition (CVPR)}, 
  title={Deep Residual Learning for Image Recognition}, 
  year={2016},
  volume={},
  number={},
  pages={770-778},
  doi={10.1109/CVPR.2016.90}}

@misc{howard2017mobilenetsefficientconvolutionalneural,
      title={MobileNets: Efficient Convolutional Neural Networks for Mobile Vision Applications}, 
      author={Andrew G. Howard and Menglong Zhu and Bo Chen and Dmitry Kalenichenko and Weijun Wang and Tobias Weyand and Marco Andreetto and Hartwig Adam},
      year={2017},
      eprint={1704.04861},
      archivePrefix={arXiv},
      primaryClass={cs.CV},
      url={https://arxiv.org/abs/1704.04861}, 
}

@misc{tan2020efficientnetrethinkingmodelscaling,
      title={EfficientNet: Rethinking Model Scaling for Convolutional Neural Networks}, 
      author={Mingxing Tan and Quoc V. Le},
      year={2020},
      eprint={1905.11946},
      archivePrefix={arXiv},
      primaryClass={cs.LG},
      url={https://arxiv.org/abs/1905.11946}, 
}

@misc{lin2020mcunettinydeeplearning,
      title={MCUNet: Tiny Deep Learning on IoT Devices}, 
      author={Ji Lin and Wei-Ming Chen and Yujun Lin and John Cohn and Chuang Gan and Song Han},
      year={2020},
      eprint={2007.10319},
      archivePrefix={arXiv},
      primaryClass={cs.CV},
      url={https://arxiv.org/abs/2007.10319}, 
}

@misc{szegedy2015rethinkinginceptionarchitecturecomputer,
      title={Rethinking the Inception Architecture for Computer Vision}, 
      author={Christian Szegedy and Vincent Vanhoucke and Sergey Ioffe and Jonathon Shlens and Zbigniew Wojna},
      year={2015},
      eprint={1512.00567},
      archivePrefix={arXiv},
      primaryClass={cs.CV},
      url={https://arxiv.org/abs/1512.00567}, 
}

@misc{liu2022convnet2020s,
      title={A ConvNet for the 2020s}, 
      author={Zhuang Liu and Hanzi Mao and Chao-Yuan Wu and Christoph Feichtenhofer and Trevor Darrell and Saining Xie},
      year={2022},
      eprint={2201.03545},
      archivePrefix={arXiv},
      primaryClass={cs.CV},
      url={https://arxiv.org/abs/2201.03545}, 
}

@article{Lin_2023,
   title={Tiny Machine Learning: Progress and Futures [Feature]},
   volume={23},
   ISSN={1558-0830},
   url={http://dx.doi.org/10.1109/MCAS.2023.3302182},
   DOI={10.1109/mcas.2023.3302182},
   number={3},
   journal={IEEE Circuits and Systems Magazine},
   publisher={Institute of Electrical and Electronics Engineers (IEEE)},
   author={Lin, Ji and Zhu, Ligeng and Chen, Wei-Ming and Wang, Wei-Chen and Han, Song},
   year={2023},
   pages={8–34} }

@ARTICLE{tinyad_10093068,
  author={Sun, Yuting and Chen, Tong and Nguyen, Quoc Viet Hung and Yin, Hongzhi},
  journal={IEEE Transactions on Industrial Informatics}, 
  title={TinyAD: Memory-Efficient Anomaly Detection for Time-Series Data in Industrial IoT}, 
  year={2024},
  volume={20},
  number={1},
  pages={824-834},
  doi={10.1109/TII.2023.3254668}}

@misc{lin2024mcunetv2memoryefficientpatchbasedinference,
      title={MCUNetV2: Memory-Efficient Patch-based Inference for Tiny Deep Learning}, 
      author={Ji Lin and Wei-Ming Chen and Han Cai and Chuang Gan and Song Han},
      year={2024},
      eprint={2110.15352},
      archivePrefix={arXiv},
      primaryClass={cs.CV},
      url={https://arxiv.org/abs/2110.15352}, 
}

@misc{stahl2023fuseddepthwisetilingmemory,
      title={Fused Depthwise Tiling for Memory Optimization in TinyML Deep Neural Network Inference}, 
      author={Rafael Stahl and Daniel Mueller-Gritschneder and Ulf Schlichtmann},
      year={2023},
      eprint={2303.17878},
      archivePrefix={arXiv},
      primaryClass={cs.LG},
      url={https://arxiv.org/abs/2303.17878}, 
}

@misc{kang2022aocstreamallonchipcnnaccelerator,
      title={AoCStream: All-on-Chip CNN Accelerator With Stream-Based Line-Buffer Architecture}, 
      author={Hyeong-Ju Kang},
      year={2022},
      eprint={2212.11438},
      archivePrefix={arXiv},
      primaryClass={cs.AR},
      url={https://arxiv.org/abs/2212.11438}, 
}

@misc{krishna2023ramanreconfigurablesparsetinyml,
      title={RAMAN: A Re-configurable and Sparse tinyML Accelerator for Inference on Edge}, 
      author={Adithya Krishna and Srikanth Rohit Nudurupati and Chandana D G and Pritesh Dwivedi and André van Schaik and Mahesh Mehendale and Chetan Singh Thakur},
      year={2023},
      eprint={2306.06493},
      archivePrefix={arXiv},
      primaryClass={cs.NE},
      url={https://arxiv.org/abs/2306.06493}, 
}

@misc{david2021tensorflowlitemicroembedded,
      title={TensorFlow Lite Micro: Embedded Machine Learning on TinyML Systems}, 
      author={Robert David and Jared Duke and Advait Jain and Vijay Janapa Reddi and Nat Jeffries and Jian Li and Nick Kreeger and Ian Nappier and Meghna Natraj and Shlomi Regev and Rocky Rhodes and Tiezhen Wang and Pete Warden},
      year={2021},
      eprint={2010.08678},
      archivePrefix={arXiv},
      primaryClass={cs.LG},
      url={https://arxiv.org/abs/2010.08678}, 
}

@misc{zhong2023memoryawareschedulingcomplexwired,
      title={Memory-aware Scheduling for Complex Wired Networks with Iterative Graph Optimization}, 
      author={Shuzhang Zhong and Meng Li and Yun Liang and Runsheng Wang and Ru Huang},
      year={2023},
      eprint={2308.13898},
      archivePrefix={arXiv},
      primaryClass={cs.LG},
      url={https://arxiv.org/abs/2308.13898}, 
}

@inproceedings{motenn_10.1145/3649329.3655922,
author = {Chen, Renze and Ding, Zijian and Zheng, Size and Li, Meng and Liang, Yun},
title = {MoteNN: Memory Optimization via Fine-grained Scheduling for Deep Neural Networks on Tiny Devices},
year = {2024},
isbn = {9798400706011},
publisher = {Association for Computing Machinery},
address = {New York, NY, USA},
url = {https://doi.org/10.1145/3649329.3655922},
doi = {10.1145/3649329.3655922},
booktitle = {Proceedings of the 61st ACM/IEEE Design Automation Conference},
articleno = {102},
numpages = {6},
location = {San Francisco, CA, USA},
series = {DAC '24}
}

@misc{han2016deepcompressioncompressingdeep,
      title={Deep Compression: Compressing Deep Neural Networks with Pruning, Trained Quantization and Huffman Coding}, 
      author={Song Han and Huizi Mao and William J. Dally},
      year={2016},
      eprint={1510.00149},
      archivePrefix={arXiv},
      primaryClass={cs.CV},
      url={https://arxiv.org/abs/1510.00149}, 
}

@misc{nagel2020downadaptiveroundingposttraining,
      title={Up or Down? Adaptive Rounding for Post-Training Quantization}, 
      author={Markus Nagel and Rana Ali Amjad and Mart van Baalen and Christos Louizos and Tijmen Blankevoort},
      year={2020},
      eprint={2004.10568},
      archivePrefix={arXiv},
      primaryClass={cs.LG},
      url={https://arxiv.org/abs/2004.10568}, 
}

@misc{hubara2020improvingposttrainingneural,
      title={Improving Post Training Neural Quantization: Layer-wise Calibration and Integer Programming}, 
      author={Itay Hubara and Yury Nahshan and Yair Hanani and Ron Banner and Daniel Soudry},
      year={2020},
      eprint={2006.10518},
      archivePrefix={arXiv},
      primaryClass={cs.LG},
      url={https://arxiv.org/abs/2006.10518}, 
}

@misc{luo2017thinetfilterlevelpruning,
      title={ThiNet: A Filter Level Pruning Method for Deep Neural Network Compression}, 
      author={Jian-Hao Luo and Jianxin Wu and Weiyao Lin},
      year={2017},
      eprint={1707.06342},
      archivePrefix={arXiv},
      primaryClass={cs.CV},
      url={https://arxiv.org/abs/1707.06342}, 
}

@misc{han2015learningweightsconnectionsefficient,
      title={Learning both Weights and Connections for Efficient Neural Networks}, 
      author={Song Han and Jeff Pool and John Tran and William J. Dally},
      year={2015},
      eprint={1506.02626},
      archivePrefix={arXiv},
      primaryClass={cs.NE},
      url={https://arxiv.org/abs/1506.02626}, 
}

@misc{he2017channelpruningacceleratingdeep,
      title={Channel Pruning for Accelerating Very Deep Neural Networks}, 
      author={Yihui He and Xiangyu Zhang and Jian Sun},
      year={2017},
      eprint={1707.06168},
      archivePrefix={arXiv},
      primaryClass={cs.CV},
      url={https://arxiv.org/abs/1707.06168}, 
}

@misc{cai2019proxylessnasdirectneuralarchitecture,
      title={ProxylessNAS: Direct Neural Architecture Search on Target Task and Hardware}, 
      author={Han Cai and Ligeng Zhu and Song Han},
      year={2019},
      eprint={1812.00332},
      archivePrefix={arXiv},
      primaryClass={cs.LG},
      url={https://arxiv.org/abs/1812.00332}, 
}

@misc{cai2020onceforalltrainnetworkspecialize,
      title={Once-for-All: Train One Network and Specialize it for Efficient Deployment}, 
      author={Han Cai and Chuang Gan and Tianzhe Wang and Zhekai Zhang and Song Han},
      year={2020},
      eprint={1908.09791},
      archivePrefix={arXiv},
      primaryClass={cs.LG},
      url={https://arxiv.org/abs/1908.09791}, 
}

@article{Alwani2016FusedlayerCA,
  title={Fused-layer CNN accelerators},
  author={Manoj Alwani and Han Chen and Michael Ferdman and Peter Milder},
  journal={2016 49th Annual IEEE/ACM International Symposium on Microarchitecture (MICRO)},
  year={2016},
  pages={1-12},
  url={https://api.semanticscholar.org/CorpusID:5804465}
}

@inproceedings{sakr2021memory,
  title     = {Memory-Efficient {CMSIS-NN} with Replacement Strategy},
  author    = {Sakr, Fouad and Bellotti, Francesco and Berta, Riccardo and De Gloria, Alessandro and Doyle, Joseph},
  booktitle = {2021 8th International Conference on Future Internet of Things and Cloud (FiCloud)},
  pages     = {299--303},
  year      = {2021},
  organization = {IEEE},
  doi       = {10.1109/FiCloud49777.2021.00050}
}

@article{Horowitz201411CE,
  title={1.1 Computing's energy problem (and what we can do about it)},
  author={Mark Horowitz},
  journal={2014 IEEE International Solid-State Circuits Conference Digest of Technical Papers (ISSCC)},
  year={2014},
  pages={10-14},
  url={https://api.semanticscholar.org/CorpusID:22028726}
}

@inproceedings{busoye2026dmc,
  title     = {Deep Microcompression: Structured Pruning and Bit-packed
               Quantization for Microcontrollers},
  author    = {Busoye, Opegbemi M. and Busoye, Tolulope M. and Eigbe, Eghonghon-aye},
  booktitle = {Proceedings of the Global South in AI Workshop at the 43rd
               International Conference on Machine Learning},
  series    = {Proceedings of Machine Learning Research},
  volume    = {306},
  year      = {2026},
  publisher = {PMLR}
}

@inproceedings{abadi2016tensorflow,
  title={TensorFlow: A system for large-scale machine learning},
  author={Abadi, Mart{\'\i}n and Barham, Paul and Chen, Jianmin and Chen, Zhifeng and Davis, Andy and Dean, Jeffrey and Devin, Matthieu and Ghemawat, Sanjay and Irving, Geoffrey and Isard, Michael and others},
  booktitle={12th USENIX symposium on operating systems design and implementation (OSDI 16)},
  pages={265--283},
  year={2016}
}
}

% \clearpage

% \maketitlesupplementary
% \appendix
% \input{supplementary/glossary}
% \input{supplementary/reprocude}
% \input{supplementary/herringbone}
% \input{supplementary/generalized_herringbone}
% \input{supplementary/dense_rip}
% \input{supplementary/validation}
% \input{supplementary/sota_layer_results}

\ifsupp
\clearpage
\maketitlesupplementary
\appendix
\section{Roadmap and Symbol Glossary}
\label{app:glossary}

This appendix carries the material deferred from the main text, in the order
the main text needs it. \cref{app:repro} states what will be released.
\cref{app:hb_prelim} fixes the herringbone debt accounting inherited
from~\cite[\S3.4]{pmlr-v97-gural19a}. \cref{app:hb_lemmas} proves the four
results behind \cref{sec:limitations}: the leg structure of the traversal,
where its peak actually falls, the closed form the paired sum collapses to,
and the exact size of the shortfall in \cref{eq:gural_eq2}.
\cref{app:corrected_hb_proof} proves \cref{cor:corrected_hb} from them.
\cref{app:generalized_hb} lifts the restrictions
of~\cite{pmlr-v97-gural19a} and derives \cref{thm:generalized_hb}, the
baseline against which RiP is measured in \cref{sec:real_layers}.
\cref{app:dense} re-derives the RiP bound when padding is not materialized.
\cref{app:validation} describes the corruption simulator and the sweeps behind
every count in \cref{sec:synthetic}, \cref{app:layer_results} and \cref{tab:sota_layer_results}
give the per-family layer results, and \cref{app:milp_scale} reports the
graph-level allocation study across 33 \texttt{keras.applications} models
(\cref{tab:milp_scale}), including the in-place upsampling bound of
\cref{app:upsample}. \Cref{tab:glossary} collects the notation.

\begin{table*}[t]
\centering
\caption{Symbols used throughout the paper, grouped by the section that
introduces them. $D$ is overloaded by arity: $D(r,c)$ is the row-major RiP debt
of \cref{eq:delta_rc}, $D(x)$ the herringbone leg debt of
\cref{eq:app_leg_debt}.}
\label{tab:glossary}
\footnotesize
\setlength{\tabcolsep}{4pt}
\begin{minipage}[t]{0.48\textwidth}
\begin{tabular}[t]{@{}ll@{}}
\toprule
\multicolumn{2}{@{}l}{\emph{Convolution setup (\cref{sec:right_align_layout})}}\\
$\mathbf{X}, \mathbf{Y}$ & Input, output activation tensors \\
$H_{in}, W_{in}, C_{in}$ & Input height, width, channel count \\
$H_{out}, W_{out}, C_{out}$ & Output height, width, channel count \\
$H_p, W_p$ & Padded input height, width \\
$N_{out}$ & Output tensor size, $H_{out}W_{out}C_{out}$ \\
$k_h, k_w$ & Kernel height, width \\
$k_h^{\mathrm{eff}}, k_w^{\mathrm{eff}}$ & Dilated effective kernel extent \\
$s_h, s_w$ & Stride, vertical and horizontal \\
$d_h, d_w$ & Dilation, vertical and horizontal \\
$p_h, p_w$ & Padding, vertical and horizontal \\
$(r, c)$ & Pixel index: $r$ row, $c$ column \\
$\gamma$ & Channel index \\
\midrule
\multicolumn{2}{@{}l}{\emph{Right In-Place (\cref{sec:right_align})}}\\
$\Delta$ & Gap size, in scalars \\
$\Delta^{*}$ & Minimum safe gap (\cref{thm:gap}) \\
$\mathcal{W}^{*}$ & Minimum workspace, $\Delta^{*} + H_pW_pC_{in}$ \\
$\mathrm{NLI}(r,c)$ & Next live input address (\cref{eq:nli}) \\
$\mathrm{out\_end}(r,c)$ & First address past the written pixel \\
$D(r, c)$ & Row-major debt at $(r,c)$ (\cref{eq:delta_rc}) \\
$\alpha, \beta$ & Column, row coefficients (\cref{eq:debt}) \\
$\sigma$ & Production-order headroom (\cref{eq:sigma}) \\
$\mathcal{R}_s, \mathcal{C}_s$ & Row, column breakpoints (\cref{eq:breakpoints}) \\
\midrule
\multicolumn{2}{@{}l}{\emph{Dense input layout (\cref{app:dense})}}\\
$(\bar{r}, \bar{c})$ & First live real input pixel (\cref{lem:dense_affine_debt}) \\
$\Delta^{*}_{\mathrm{d}}$ & Minimum gap, dense (\cref{thm:gap_dense}) \\
$\mathcal{W}^{*}_{\mathrm{d}}$ & $\Delta^{*}_{\mathrm{d}} + H_{in}W_{in}C_{in}$ \\
$L$ & Dead leading run, in pixels (\cref{app:dense_gcd}) \\
\bottomrule
\end{tabular}
\end{minipage}
\hfill
\begin{minipage}[t]{0.48\textwidth}
\begin{tabular}[t]{@{}ll@{}}
\toprule
\multicolumn{2}{@{}l}{\emph{Herringbone, restricted (\cref{app:hb_prelim})}}\\
$W$ & Square output side, $W \equiv W_{out} = H_{out}$ \\
$D(x)$ & Leg debt, length $x$ (\cref{eq:app_leg_debt}) \\
$D_{row}(x), D_{col}(x)$ & Row, column leg debt (\cref{eq:heb_sequence}) \\
$A(x), B(x)$ & Output, input parts: $D = A - B$ \\
$L_D(x^0)$ & Running debt, legs $\ge x^0$ (\cref{eq:app_running_parabola}) \\
$D_{hb}$ & Peak debt of the traversal (\cref{eq:app_true_peak}) \\
$D_{hb}^{\mathrm{G\&M}}$ & As-published peak debt (\cref{eq:gural_eq2}) \\
$\alpha_{hb}$ & Leg-debt slope, $C_{out}-C_{in}$ \\
$\rho$ & Per-leg residual input, $(k-1)C_{in}$ \\
$x^{*}$ & Critical leg, $D(x^{*}) = 0$ (\cref{eq:app_xstar}) \\
$\hat{x}$ & $\min(\lfloor x^{*}\rfloor, W_{out}{-}1)$ (\cref{eq:correction_limit2}) \\
$\Phi(\lfloor x^{*}\rfloor)$ & Paired-sum closed form (\cref{eq:app_closed_form}) \\
\midrule
\multicolumn{2}{@{}l}{\emph{Herringbone, generalized (\cref{app:generalized_hb_setup})}}\\
$M, n$ & $\max$, $\min$ of $H_{out}, W_{out}$ \\
$\delta$ & Axis-length difference, $M - n$ \\
$w_r, h_r$ & Residual footprint per axis (\cref{eq:alpha_rho}) \\
$\alpha_{hb}^{gen}$ & $C_{out} - s_hs_wC_{in}$ (\cref{eq:alpha_rho}) \\
$\rho_r, \rho_c$ & Per-leg residual, horizontal, vertical \\
$R_r, R_c$ & Remainder at a leg tail (\cref{eq:R_axis}) \\
$\rho_{maj}, \rho_{min}$ & $\rho_r, \rho_c$ by longer axis \\
$R_{maj}, R_{min}$ & $R_r, R_c$ by longer axis \\
$\beta_{hb}$ & Last-leg slope, $C_{out} - R_{min}$ \\
$D_{maj}, D_{min}, D_{maj}^{\mathrm{last}}$ & Three leg types (\cref{eq:leg_major,eq:leg_minor,eq:leg_last}) \\
$L_D^{maj}, L_D^{min}$ & Across-leg running debt per axis \\
$x^{*}_{maj}, x^{*}_{min}$ & Vertex per axis (\cref{eq:app_xstar_maj,eq:app_xstar_min}) \\
$\mathcal{C}_{maj}, \mathcal{C}_{min}$ & Candidate leg lengths (\cref{thm:generalized_hb}) \\
$L_{\mathrm{last}}$ & Last-leg candidate (\cref{eq:Ld_last}) \\
$D_{hb}^{\mathrm{gen}}$ & Generalized peak debt (\cref{eq:generalized_hb}) \\
\bottomrule
\end{tabular}
\end{minipage}
\end{table*}
\section{Reproduction}
\label{app:repro}

Three artifacts will be released: the closed forms of
\cref{thm:gap,thm:gap_dense,thm:generalized_hb} together with the corruption
simulator and the sweep script behind \cref{tab:sweeps}; the layer set and
comparison script behind \cref{tab:sota_layer_results}; and the modified
TinyEngine, comprising the RiP kernels, the MILP allocator of
\cref{sec:execution}, and the firmware used for
\cref{tab:deployment}.\footnote{Repository withheld for double-blind review;
the URL replaces this note in the camera-ready version.}
\section{Herringbone Preliminaries}
\label{app:hb_prelim}

We adopt the debt accounting of~\cite[\S3.4]{pmlr-v97-gural19a}, specialized to
their restricted setting: valid padding, unit stride, unit dilation, and an odd
square $k \times k$ kernel, so $H_{out} = H_{in}-(k-1)$ and
$W_{out} = W_{in}-(k-1)$. Throughout this appendix $W \equiv W_{out} = H_{out}$,
$\alpha_{hb} \equiv C_{out}-C_{in}$ and $\rho \equiv (k-1)C_{in}$.

The herringbone traversal processes output pixels in an alternating row--column
order, reclaiming edge input pixels as they go stale
(\cref{fig:herringbone_squ}). A leg of length $x$ writes $x C_{out}$ scalars and
frees $\big(x+(k-1)\big)C_{in}$ per stale input pixel, less the $k-1$ extra input pixels along
the current row or column arising from the input--output size difference,
giving net debt

\begin{equation}
D(x) = x\,\alpha_{hb} - \rho.
\label{eq:app_leg_debt}
\end{equation}
where $\alpha_{hb} \equiv C_{out}-C_{in}$, $\rho \equiv (k-1)C_{in}$.

% Writing $L_D$ for the running sum over successive legs, the peak debt is

% \begin{equation}
% D_{hb} = \max_i \sum_{j \le i} L_D(j) + k C_{in},
% \label{eq:app_running_sum}
% \end{equation}

Writing $\ell_1, \ell_2, \dots$ for the successive leg lengths and $L_D(x^0)$
for the running debt once every leg of length at least $x^0$ has been processed,
the peak debt is
\begin{equation}
D_{hb} = \max_{x^0} L_D(x^0), \qquad
L_D(x^0) = \sum_{\ell_j \,\ge\, x^0} D(\ell_j) + kC_{in},
\end{equation}

the trailing $kC_{in}$ is the residual input footprint needed for the last
output pixel. Since $\alpha_{hb} > 0$ makes $D$ increasing in $x$, the running
sum grows while leg debts are positive and shrinks once they turn negative, so
the maximum falls where $D(x^*) = 0$:

\begin{equation}
x^* = \frac{\rho}{\alpha_{hb}} = \frac{(k-1)C_{in}}{C_{out}-C_{in}}.
\label{eq:app_xstar}
\end{equation}

\section{Herringbone Lemmas}
\label{app:hb_lemmas}

We collect the four results invoked in \cref{sec:limitations} and
\cref{sec:corrected}. \cref{lem:leg_pairing} establishes the traversal
structure; \cref{lem:concavity} locates the peak within it;
\cref{lem:closed_form_equivalence} converts the resulting sum into closed form;
\cref{lem:flooring_error} then measures exactly what the published debt of \cite{pmlr-v97-gural19a} includes.

\begin{lemma}[Leg pairing under square output]
\label{lem:leg_pairing}
Under the restrictions of \cref{app:hb_prelim} with square output
$H_{out} = W_{out} = W$, the herringbone leg sequence consists of paired legs of matching length

\begin{equation}
\begin{split}
D_{row}(W_{out}),\, &\big[D_{col}(W_{out}{-}1),\;
 D_{row}(W_{out}{-}1)\big], \\  &\dots, \big[D_{col}(1)\;, D_{row}(1) \big],\
\end{split}
\label{eq:heb_sequence}
\end{equation}

with the first row leg of length $W_{out}$ unpaired (\cref{fig:herringbone_squ}). The
running sum of $D$-values down to the leg with length $x^0$ is therefore
\begin{equation}
D(W_{out}) + 2\sum_{x=x^0}^{W_{out}-1}{D(x)} + kC_{in}.
\label{eq:app_paired_sum}
\end{equation}
\end{lemma}

\begin{proof}

The traversal alternates (row-leg, column-leg, row-leg, column-leg $\dots$, row-leg) processing edges in decreasing length after the first sweep of the outer row. Under $H_{out} = W_{out}$, each $i$th column-leg and the $(i{+}1)$th row leg have the same length for $i \in [1, W_{out}{-}1]$, so summing over legs of length at least $x^0$ counts each length twice, except for the first row leg, which has a length of $W_{out}$.
\end{proof}

\begin{lemma}[Where the peak debt occurs]
\label{lem:concavity}
Let $x^*$ be as in \cref{eq:app_xstar}. Under $\alpha_{hb}>0$ the running sum
$L_D$ is a strictly concave downward parabola with vertex
$x = \tfrac{1}{2} + \rho/\alpha_{hb}$, and the peak debt of the
traversal is attained after both legs of length $\lfloor x^* \rfloor + 1$ have been processed:
\begin{equation}
D_{hb} = D(W_{out}) + 2\!\!\sum_{x=\lfloor x^* \rfloor + 1}^{W_{out}-1}\!\! D(x) + kC_{in}.
\label{eq:app_true_peak}
\end{equation}
\end{lemma}

\begin{proof}
Substituting \cref{eq:app_leg_debt} into $L_D(x^0) = D(W_{out}) + 2\sum_{x=x^0}^{W_{out}-1}
D(x) + kC_{in}$, the debt if the traversal stops after both legs of length
$x^0$, collects to~\footnote{Under the generalized herringbone traversal, this is a special case of \cref{eq:Ld_minor} with $y=1$, just about to process the first pixel of the next column; where $\rho_{maj} = \rho_{min} = \rho$}

\begin{equation}
\begin{split}
L_D(x^0) = &-\alpha_{hb}(x^0)^2 + (\alpha_{hb} + 2\rho)\,x^0 \\
          & + \big(\alpha_{hb}W_{out}^2 - 2\rho W_{out} - \rho + kC_{in}\big).
\label{eq:app_running_parabola}
\end{split}
\end{equation}

The second derivative is $-2\alpha_{hb} < 0$, and $L_D'(x^0) = 0$ at
$x^0 = \tfrac{1}{2} + \rho/\alpha_{hb} = x^* + \tfrac{1}{2}$: the running
sum peaks half a leg past the length at which the leg debt changes sign. A
concave parabola is maximized over the integers at the integer nearest its
vertex, and that integer is $\lfloor x^* \rfloor + 1$;
the clamp to $W_{out}$ covers the case $x^* \ge W_{out}$, in which no traversed leg has
positive debt: the running sum never grows. Therefore, the peak is reached after the first
leg alone and the sum in \cref{eq:app_true_peak} is empty, leaving
$D_{hb} = D(W_{out}) + kC_{in}$.

When $x^*$ is an integer the nearest integer is
tied between $x^*$ and $x^*+1$, and either gives the same value since
$D(x^*) = 0$. When $x^*$ is not an integer and with the clamp inactive, $\lfloor x^* \rfloor$ + 1 is the shortest leg with positive debt and
$\lfloor x^* \rfloor$, the longest with negative debt:

\begin{equation}
D(\lfloor x^* \rfloor + 1) > 0 > D(\lfloor x^* \rfloor).
\label{eq:app_bracket}
\end{equation}
By \cref{lem:leg_pairing} each leg of length $x$ is followed by its
partner, so the peak includes both legs of length $\lfloor x^* \rfloor + 1$ and stops there.

% What the published form cannot express is not this maximum but the pairing.
% By \cref{rem:unpaired_lower}, $\Phi(x^0) = L_D(x^0{+}1) + D(x^0)$: evaluating
% the telescoped identity at any integer leaves $D(x^0)$ unpaired at the lower
% bound. No choice of $x^0$ removes that term, which is the root of
% Limitation~1.
\qed
\end{proof}

\begin{lemma}[Closed-form equivalence of peak debt]
\label{lem:closed_form_equivalence}
G\&M~\cite{pmlr-v97-gural19a} evaluate the debt at the square of side $\lfloor x^* \rfloor$, so their expression accounts for the legs from $W_{out}$ down to the first leg of length $\lfloor x^* \rfloor$. For the integer $\lfloor x^* \rfloor\in [1, W_{out}]$ that accounting is $\Phi(\lfloor x^* \rfloor)$, where

\begin{equation}
\begin{split}
\Phi(\lfloor x^* \rfloor) &\;\equiv\;
C_{out}\big(W_{out}^2 - {\lfloor x^* \rfloor}^2\big) \\
&- C_{in}\Big(W_{in}^2 - ({\lfloor x^* \rfloor}{+}k{-}1)^2 - k\Big)
\end{split}
\label{eq:app_closed_form}
\end{equation}
\end{lemma}

\begin{proof}
Group the legs into adjacent pairs up to the first leg of length $\lfloor x^* \rfloor$. We can write the accounting as
$\sum_{x=\lfloor x^* \rfloor}^{W_{out}-1}\big[D(x{+}1)+D(x)\big]$ which counts each length in
$(\lfloor x^* \rfloor, W_{out})$ twice and each endpoint, $W_{out}$ and $\lfloor x^* \rfloor$, once. Split $D$
into its output and input parts, $D(x) = A(x) - B(x)$ with $A(x) = C_{out}x$
and $B(x) = C_{in}(x{+}k{-}1)$, and sum each.

The output part contributes $A(x{+}1) + A(x) = (2x{+}1)C_{out}$, so
\begin{equation}
\label{eq:pair_sum_A}
\begin{split}
\sum_{x=\lfloor x^* \rfloor}^{W_{out}-1} &\big[A(x{+}1)+A(x)\big] \\
&= C_{out}\Big[2\!\!\sum_{x=\lfloor x^* \rfloor}^{W_{out}-1}\!\! x
   \;+\; (W_{out} - \lfloor x^* \rfloor)\Big] \\
&= C_{out}(W_{out} - \lfloor x^* \rfloor)
   \big[(W_{out} - 1 + \lfloor x^* \rfloor) + 1\big] \\
&= C_{out}\big(W_{out}^{2} - \lfloor x^* \rfloor^{2}\big).
\end{split}
\end{equation}

The input part contributes $B(x{+}1)+B(x) = (2x + 2k - 1)C_{in}$, so

\begin{equation}
\label{eq:pair_sum_B}
\begin{split}
\sum_{x=\lfloor x^* \rfloor}^{W_{out}-1} &\big[B(x{+}1)+B(x)\big] \\
&= C_{in}\Big[2\!\!\sum_{x=\lfloor x^* \rfloor}^{W_{out}-1}\!\! x
   \;+\; (2k{-}1)(W_{out} - \lfloor x^* \rfloor)\Big] \\
&= C_{in}(W_{out} - \lfloor x^* \rfloor)
   \big[(W_{out} - 1 + \lfloor x^* \rfloor) + (2k{-}1)\big] \\
&= C_{in}\big((W_{out}{+}k{-}1)^2 - (\lfloor x^* \rfloor{+}k{-}1)^{2}\big).
\end{split}
\end{equation}

the last step by difference of squares, since
$W_{out} - \lfloor x^* \rfloor = (W_{out}{+}k{-}1) - (\lfloor x^* \rfloor{+}k{-}1)$ and
$W_{out} - 1 + \lfloor x^* \rfloor + 2k - 1 = (W_{out}{+}k{-}1) + (\lfloor x^* \rfloor{+}k{-}1)$.

Subtracting \cref{eq:pair_sum_B} from \cref{eq:pair_sum_A}, substituting
$W_{in} = W_{out} + k - 1$, and adding  the trailing $+kC_{in}$ for the preceding output of
\cref{eq:app_paired_sum} into the $C_{in}$ bracket, which changes its sign,
gives \cref{eq:app_closed_form}.\qed
\end{proof}

\begin{remark}
\label{rem:unpaired_lower}
Expanding \cref{eq:app_paired_sum} shows both endpoints contribute \emph{once}:
\[
\sum_{x=\lfloor x^* \rfloor}^{W_{out}-1}\big[D(x{+}1)+D(x)\big]
= D(W_{out}) + 2\!\!\sum_{x=\lfloor x^* \rfloor+1}^{W_{out}-1}\!\! D(x) + D(\lfloor x^* \rfloor).
\]
The single count of $D(W_{out})$ is genuine: that leg occurs once in the
traversal. The single count of $D(\lfloor x^* \rfloor)$ is not, and it places $\Phi$ exactly
halfway between two consecutive attainable states,
$\Phi(\lfloor x^* \rfloor) = \tfrac{1}{2}\big[L_D(\lfloor x^* \rfloor{+}1) + L_D(\lfloor x^* \rfloor)\big]$: the debt therefore includes an extra negative debt which reduces the computed debt. No choice of
$x$ removes that term, which is what fixes the coefficient in
\cref{lem:flooring_error}. The published choice $\lfloor x^* \rfloor$ is correct for
its own parameterization; what it misses is the pairing, not the location of
the peak.
\end{remark}

\begin{lemma}[Exact flooring underestimate]
\label{lem:flooring_error}
When $x^*$ is non-integer and $\lfloor x^* \rfloor \le W_{out}-1$, the as-published herringbone
formula (\cref{eq:gural_eq2}) underestimates the true peak debt by exactly
\begin{equation}
D_{hb} - D_{hb}^{\mathrm{G\&M}} = (k-1)C_{in} \bmod (C_{out}-C_{in}).
\label{eq:flooring_gap}
\end{equation}
\end{lemma}

\begin{proof}
\cref{eq:gural_eq2} is $\Phi(\lfloor x^*\rfloor) = D_{hb}^{\mathrm{G\&M}}$. By
\cref{rem:unpaired_lower},
\[
\Phi(\lfloor x^* \rfloor)= D(W_{out}) + 2\!\!\sum_{x=\lfloor x^* \rfloor+1}^{W_{out}-1}\!\! D(x) + D(\lfloor x^* \rfloor) + kC_{in} ,
\]
which by \cref{eq:app_true_peak} is $D_{hb} + D(\lfloor x^* \rfloor)$, it contains one additional
term. Since $D(\lfloor x^* \rfloor) < 0$ by
\cref{eq:app_bracket}, the published form falls short:
\[
D_{hb} - D_{hb}^{\mathrm{G\&M}} = -D(\lfloor x^* \rfloor) = \rho - \alpha_{hb}\lfloor x^* \rfloor.
\]
With $\lfloor x^* \rfloor = \lfloor \rho/\alpha_{hb}\rfloor$, the definition of the modular
remainder gives $\rho - \lfloor \rho/\alpha_{hb}\rfloor\alpha_{hb}
= \rho \bmod \alpha_{hb}$, which is \cref{eq:flooring_gap}.

When $\lfloor x^*\rfloor \ge W_{out}$ no traversed leg has
positive debt, the clamp sets $\hat{x} = W_{out}-1$, and the sum in
\cref{eq:app_true_peak} is empty. When $\rho < \alpha_{hb}$, we have
$\lfloor x^* \rfloor = 0$ and the sum runs from $1$, which is a genuine leg; no special
case arises. The arithmetic is unchanged, and this
is the regime of the first convolutional layer in
\cref{tab:limitations}.\qed
\end{proof}

\subsection{Proof of \cref{cor:corrected_hb}}
\label{app:corrected_hb_proof}

\begin{proof}
If $\alpha_{hb} \le 0$ then $D(x) \le 0$ for all $x \ge 0$, so the running sum
never grows and the peak is the write-before-clear headroom $C_{out}$ alone.

% For $\alpha_{hb} > 0$, by \cref{lem:flooring_error},
% $D_{hb} = \Phi(\lfloor x^* \rfloor) - D(\lfloor x^* \rfloor) = \Phi(\lfloor x^* \rfloor) + (\rho - \alpha_{hb}\lfloor x^* \rfloor)$. Writing $\Phi(\lfloor x^*\rfloor)$ out in full is:

% \begin{equation*}
% \begin{split}
% D_{hb} &= C_{out}\big(H_{out}W_{out} - \hat{x}^2\big) \\
%        &\quad - C_{in}\big(H_{in}W_{in} - (\hat{x}+k-1)^2 - k\big) \\
%        &\quad + (\rho - \alpha_{hb}\hat{x}).
% \end{split}
% \end{equation*}

% The clamp $\hat{x} = \min\!\big(\lfloor x^* \rfloor,\; W_{out}-1\big)$ follows from \cref{lem:concavity}.\qed

For $\alpha_{hb} > 0$, \cref{lem:flooring_error} gives
$D_{hb} = \Phi(\lfloor x^* \rfloor) - D(\lfloor x^* \rfloor)$, which is
$\Phi(\lfloor x^* \rfloor) + (\rho - \alpha_{hb}\lfloor x^* \rfloor)$.
By \cref{lem:concavity} the peak is attained after both legs of length
$\min(\lfloor x^* \rfloor + 1,\, W_{out})$, so the evaluation point is one leg
shorter:
\[
\hat{x} \;\equiv\; \min(\lfloor x^* \rfloor + 1,\, W_{out}) - 1
        \;=\; \min(\lfloor x^* \rfloor,\, W_{out} - 1).
\]
Writing $\Phi$ out in full from \cref{eq:app_closed_form} at $\hat{x}$, with
$H_{out} = W_{out}$ and $H_{in} = W_{in}$, gives \cref{eq:corrected_hb}:
\begin{equation*}
\begin{split}
D_{hb} &= C_{out}\big(H_{out}W_{out} - \hat{x}^2\big) \\
       &\quad - C_{in}\big(H_{in}W_{in} - (\hat{x}+k-1)^2 - k\big) \\
       &\quad + (\rho - \alpha_{hb}\hat{x}).
\end{split}
\end{equation*}

\qed
\end{proof}

\section{Generalized Herringbone: Full Derivation}
\label{app:generalized_hb}

\cref{app:hb_prelim} assumed valid padding, unit stride, unit dilation, an odd
square kernel, and square output. Lifting these breaks the row--column
symmetry, changing both the per-leg debt and the leg sequence. We state the
generalized debt equation, then locate its peak.

\subsection{Axis-Agnostic Setup}
\label{app:generalized_hb_setup}

Let $M = \max(H_{out}, W_{out})$, $n = \min(H_{out}, W_{out})$, and
$\delta = M - n$. Let $w_r = W_p - s_w(W_{out}-1)$ and
$h_r = H_p - s_h(H_{out}-1)$ be the residual input footprints on each axis.
Define the leg-debt slope and the per-leg residual inputs

\begin{equation}
\begin{aligned}
\alpha_{hb}^{gen} &\equiv C_{out} - s_h s_w C_{in}, \\ \qquad
\rho_r &= s_h C_{in}(w_r - s_w), \;
\rho_c = s_w C_{in}(h_r - s_h),
\end{aligned}
\label{eq:alpha_rho}
\end{equation}

together with the remainder each leg releases at its tail,

\begin{equation}
\begin{split}
R_r &= w_r s_h C_{in} = \rho_r + s_hs_wC_{in}, \\ \qquad
R_c &= h_r s_w C_{in} = \rho_c + s_hs_wC_{in},
\end{split}
\label{eq:R_axis}
\end{equation}

where the $r$ subscript names a horizontal leg and $c$ a vertical one. Both
pairs are then relabeled by which axis is longer: when $W_{out} \ge H_{out}$
the major axis is horizontal and
$(\rho_{maj}, \rho_{min}, R_{maj}, R_{min}) = (\rho_r, \rho_c, R_r, R_c)$;
otherwise the assignment is reversed. Write
$\beta_{hb} \equiv C_{out} - R_{min}$ for the last-leg slope. Under the
restrictions of \cref{app:hb_prelim} these collapse to
$\alpha_{hb}^{gen} = C_{out}-C_{in}$ and $\rho_{maj} = \rho_{min} = \rho$, recovering
\cref{eq:app_leg_debt}.

\paragraph{Traversal for uneven output grids.}
When $H_{out}\neq W_{out}$ the traversal begins along the major axis and
interleaves the two axes thereafter, alternating a major-axis leg with a
shorter minor-axis leg and always ending on the major axis
(\cref{fig:combined_herringbone}):

\begin{equation}
\begin{aligned}
&\underbrace{\mathrm{maj}(M),\;\mathrm{min}(n{-}1)}_{\text{pair 1}},\;
\underbrace{\mathrm{maj}(M{-}1),\;\mathrm{min}(n{-}2)}_{\text{pair 2}},\;\ldots, \\
&\underbrace{\mathrm{maj}(M{-}n{+}2),\;\mathrm{min}(1)}_{\text{pair }n{-}1},\;
\mathrm{maj}(M{-}n{+}1).
\end{aligned}
\label{eq:uneven_sequence}
\end{equation}

Here $\mathrm{maj}$ and $\mathrm{min}$ denote a leg along the major and minor
axis: $\mathrm{maj}=\mathrm{row}$, $\mathrm{min}=\mathrm{col}$ when the major
axis is horizontal ($W_{out}\ge H_{out}$, \cref{fig:row_major}), and the reverse
when it is vertical ($H_{out}> W_{out}$, \cref{fig:col_major}). For the
$5\times 3$ output of \cref{fig:col_major} ($M=5$, $n=3$, $\delta=2$) the
sequence is $\mathrm{col}(5),\ \mathrm{row}(2),\ \mathrm{col}(4),\
\mathrm{row}(1),\ \mathrm{col}(3)$.

The extended debt sequence $L_D$ concatenates the per-step debts in this order,
and the peak debt is

\begin{equation}
D_{hb}^{\mathrm{gen}} = \max_i \sum_{j=1}^{i} L_D(j) + R,
\label{eq:gen_running_sum}
\end{equation}

where $R$ is $R_r$ or $R_c$ according to the axis of the leg containing step
$j$. \cref{lem:traversal_complete} confirms that the sequence visits every
output pixel exactly once, so the peak is taken over a genuine traversal of the
full grid.

\begin{figure}[htbp]
    \centering
    \begin{subfigure}[b]{0.48\textwidth}
        \centering
        \includegraphics[width=\linewidth]{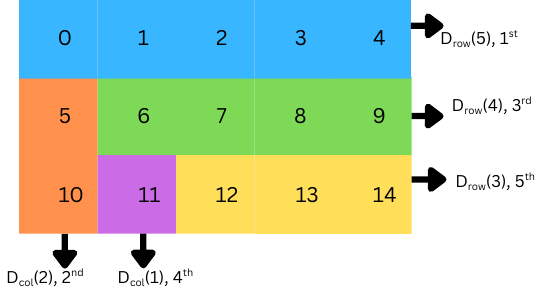}
        \caption{Row-major herringbone order}
        \label{fig:row_major}
    \end{subfigure}
    \hfill
    \begin{subfigure}[b]{0.4\textwidth}
        \centering
        \includegraphics[width=\linewidth]{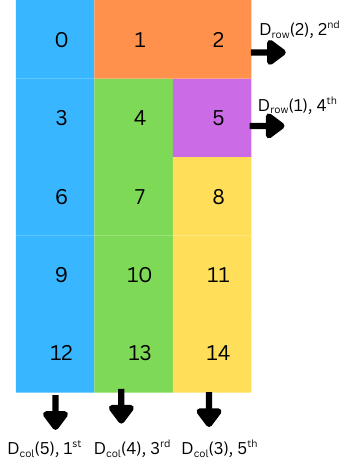}
        \caption{Column-major herringbone order}
        \label{fig:col_major}
    \end{subfigure}
    \caption{Herringbone traversal order on uneven output grids.}
    \label{fig:combined_herringbone}
\end{figure}

\begin{lemma}[Traversal completeness]
\label{lem:traversal_complete}
The herringbone traversal of an $M \times n$ output grid visits every output
pixel exactly once; equivalently, the leg lengths sum to $Mn$.
\end{lemma}

\begin{proof}
By \cref{eq:uneven_sequence} the traversal takes $n$ major-axis legs of lengths
$M, M{-}1, \dots, M{-}n{+}1$ and $n-1$ minor-axis legs of lengths
$n{-}1, n{-}2, \dots, 1$, so

\begin{align}
\sum_{\text{legs}} \text{length}
&= \underbrace{\sum_{x=M-n+1}^{M} x}_{\text{major}}
 + \underbrace{\sum_{x=1}^{n-1} x}_{\text{minor}} \\
&= \tfrac{n(2M-n+1)}{2} + \tfrac{(n-1)n}{2}
 = \tfrac{n}{2}(2M) = Mn. \nonumber
\end{align}

Each leg covers a fresh row or column edge, so no pixel is revisited; a
traversal of $Mn$ steps that never revisits a pixel visits every one of the
$Mn$ pixels exactly once.\qed
\end{proof}

\subsection{The Debt Equation Has Three Leg Types}
\label{app:hb_debt_equation}

Index a leg by the length the number of pixels $x$ in its major-axis partner\footnote{The minor axis
is shorter by $\delta+1$, so a minor leg indexed $x$ has actual length
$x-\delta-1$.} and the position of the pixel within it $y$. The debt then takes one of three
forms:

\begin{equation}
\begin{gathered}
\text{major-axis leg:}\quad D_{maj}(x,y) = \alpha_{hb}^{gen}\, y - \rho_{maj}, \\
x \in [\delta+2, M], \quad y \in [1, x]
\end{gathered}
\label{eq:leg_major}
\end{equation}

\begin{equation}
\begin{gathered}
\text{minor-axis leg:}\quad D_{min}(x,y) = \alpha_{hb}^{gen}\, y - \rho_{min}, \\
x \in [\delta+2, M], \quad y \in [1, x-\delta-1]
\end{gathered}
\label{eq:leg_minor}
\end{equation}

\begin{equation}
\begin{gathered}
\text{last major leg ($x=\delta+1$):}\quad D_{maj}^{\mathrm{last}}(y) = C_{out} + \beta_{hb}(y-1), \\
y \in [1, \delta+1]
\end{gathered}
\label{eq:leg_last}
\end{equation}

The last leg does not follow the alternating pattern and is governed by
$R_{min}$ rather than by $s_h s_w C_{in}$: it releases the remainder left at the tail of
the traversal, which is why its slope is $\beta_{hb}$ and not $\alpha_{hb}^{gen}$.

\paragraph{Locating the peak.}
\label{app:hb_locating_peak}
The total debt after a prefix of legs is the running sum $L_D = \sum D$. The
three leg types contribute differently, so $L_D$ is piecewise quadratic and the
global peak may fall in any one component; we evaluate a candidate maximum for
each and take the overall maximum (\cref{eq:generalized_hb}).

Within a single leg the debt is linear in $y$ with slope $\alpha_{hb}^{gen}$
(\cref{eq:leg_major,eq:leg_minor}), so the within-leg maximum is fixed by the
sign of $\alpha_{hb}^{gen}$ alone: for $\alpha_{hb}^{gen} > 0$ the debt increases along the
leg and the maximum is at its end, $y = x$ on the major axis or
$y = x-\delta-1$ on the minor; for $\alpha_{hb}^{gen} \le 0$ it is non-increasing and
the maximum is at the start, $y = 1$.

\paragraph{Across-leg maximum, major axis.}
\label{app:generalized_hb_across_major}
Summing $D_{maj}$ and $D_{min}$ over all legs of length $\ge x$, with the
within-leg contribution resolved as above,

\begin{align}
L_D^{maj}(x, y) &= \sum_{i=x+1}^{M} \big( D_{maj}(i) + D_{min}(i) \big) \nonumber \\
&\quad + \alpha_{hb}^{gen}(y - 1) + C_{out}.
\label{eq:Ld_major}
\end{align}

Collecting terms as in \cref{app:corrected_hb_proof},

\begin{align}
L_D^{maj}(x, y) = {}& -\alpha_{hb}^{gen} x^2 + (\rho_{maj} + \rho_{min} + \alpha_{hb}^{gen}\delta)x \nonumber \\
&+ \big[\alpha_{hb}^{gen} M^2 - M(\rho_{maj} + \rho_{min} + \alpha_{hb}^{gen}\delta) \nonumber \\
&\quad + \alpha_{hb}^{gen}(y - 1) + C_{out}\big],
\label{eq:Ld_major_closed}
\end{align}

a downward parabola in $x$ when $\alpha_{hb}^{gen} > 0$. Setting
$\partial L_D^{maj}/\partial x = 0$,

\begin{equation}
x^*_{maj} = \frac{\rho_{maj} + \rho_{min}}{2\alpha_{hb}^{gen}} + \frac{\delta}{2}.
\label{eq:app_xstar_maj}
\end{equation}

By \cref{lem:concavity} the integer maximizer is $\lfloor x^*_{maj} \rfloor$ or
$\lceil x^*_{maj} \rceil$, clipped to the leg-length domain $[\delta+2, M]$;
these populate $\mathcal{C}_{maj}$. When $\alpha_{hb}^{gen} \le 0$ the parabola opens
upward and the maximum is pushed to a domain endpoint, so
$\mathcal{C}_{maj} = \{\delta+2, M\}$.

\paragraph{Across-leg maximum, minor axis.}
\label{app:generalized_hb_across_minor}
The minor-axis sum ends one leg later, after its paired major leg has just completed:

\begin{align}
L_D^{min}(x, y) &= \sum_{i=x+1}^{M} \big( D_{maj}(i) + D_{min}(i) \big) \nonumber \\
&\quad + D_{maj}(x) + \alpha_{hb}^{gen}(y - 1) + C_{out},
\label{eq:Ld_minor}
\end{align}

which collects to \cref{eq:Ld_minor_closed}, differing from
\cref{eq:Ld_major_closed} only in the linear coefficient, where $\delta$
becomes $\delta+1$, and in an extra $-\rho_{maj}$ in the constant.

\begin{align}
L_D^{min}(x, y) = {}& -\alpha_{hb}^{gen} x^2 + \big(\rho_{maj} + \rho_{min} + \alpha_{hb}^{gen}(\delta+1)\big)x \nonumber \\
&+ \big[\alpha_{hb}^{gen} M^2 - M(\rho_{maj} + \rho_{min} + \alpha_{hb}^{gen}\delta) \nonumber \\
&\quad - \rho_{maj} + \alpha_{hb}^{gen}(y - 1) + C_{out}\big],
\label{eq:Ld_minor_closed}
\end{align}

Differentiating gives

\begin{equation}
x^*_{min} = \frac{\rho_{maj} + \rho_{min}}{2\alpha_{hb}^{gen}} + \frac{\delta+1}{2}
= x^*_{maj} + \tfrac{1}{2},
\label{eq:app_xstar_min}
\end{equation}

exactly half a leg to the right of the major-axis vertex --- a structural
consequence of the later summation start, independent of the convolution
parameters. $\mathcal{C}_{min}$ is formed from $x^*_{min}$ as
$\mathcal{C}_{maj}$ was from $x^*_{maj}$.

\paragraph{Last-leg candidate.}
\label{app:generalized_hb_lastleg}
The leg at $x = \delta+1$ closes the traversal and is governed by $\beta_{hb}$
(\cref{eq:leg_last}). Summing its contribution onto the completed minor
sequence,

\begin{equation}
L_{\mathrm{last}}
= L_D^{min}(\delta+2, 1) - R_{min} + C_{out} + \max(0, \beta_{hb}\,\delta),
\label{eq:Ld_last}
\end{equation}

where $L_D^{min}(\delta+2, 1) - R_{min}$ is the debt accumulated just before
the last leg starts, evaluated at the final minor leg with one pixel. The
$\max(0, \beta_{hb}\delta)$ term captures whether the accumulated last-leg debt
exceeds the $R_{min}$ payback; when $\beta_{hb} \le 0$ the payback dominates
and the term vanishes.

\begin{theorem}[Generalized herringbone peak debt]
\label{thm:generalized_hb}
Under arbitrary stride $(s_h, s_w)$, dilation $(d_h, d_w)$, padding
$(p_h, p_w)$, rectangular kernel $k_h \times k_w$, and possibly non-square
output, the peak debt of the herringbone traversal is
\begin{equation}
\begin{aligned}
D_{hb}^{\mathrm{gen}} \;=\; \max\Big(\, &C_{out},\;\; L_{\mathrm{last}},\;\; \max_{x\in\mathcal{C}_{\mathrm{maj}}} L_D^{maj}(x), \\
&\max_{x\in\mathcal{C}_{\mathrm{min}}} L_D^{min}(x) \,\Big),
\end{aligned}
\label{eq:generalized_hb}
\end{equation}
where $\mathcal{C}_{maj}$ and $\mathcal{C}_{min}$ are the candidate leg lengths
of \cref{app:generalized_hb_across_major,app:generalized_hb_across_minor}:
$\{\lfloor x^* \rfloor, \lceil x^* \rceil\}$ clipped to $[\delta+2, M]$ when
$\alpha_{hb}^{gen} > 0$, and the boundary points $\{\delta+2, M\}$ when
$\alpha_{hb}^{gen} \le 0$. Every argument of the maximum is available in closed form,
so the evaluation is $O(1)$.
\end{theorem}

\begin{proof}
For each leg type, substitute the debt equation
(\cref{eq:leg_major,eq:leg_minor,eq:leg_last}) into
\cref{eq:gen_running_sum} and collect, as carried out in
\cref{app:generalized_hb_across_major,app:generalized_hb_across_minor,app:generalized_hb_lastleg}.
The within-leg maximum is at $y=x$ when $\alpha_{hb}^{gen}>0$ and $y=1$ when
$\alpha_{hb}^{gen}\le 0$ (\cref{app:hb_locating_peak}), and the across-leg vertex
follows by differentiating the resulting parabola and applying
\cref{lem:concavity}. Since $L_D$ is piecewise quadratic with one component per
leg type, its global maximum is the largest of the three component maxima,
together with the floor $C_{out}$ that the write-before-clear headroom imposes.
The degenerate case is treated in
\cref{app:generalized_hb_degenerate}.\qed
\end{proof}

\subsection{Degenerate Case and Scope}
\label{app:generalized_hb_degenerate}

\paragraph{Single-row or single-column output.}
When $n = \min(H_{out}, W_{out}) = 1$ the minor axis has one pixel across, so by the
pairing offset, no minor axis leg exists: the minor leg length $x-\delta-1$ is
non-positive for every $x \le M$. Since $M = \delta+1$, only the last leg
survives on the major axis; the alternating sum collapses to the last
leg contribution, and
$\mathcal{C}_{maj} \cup \mathcal{C}_{min} = \varnothing$, so
\cref{eq:generalized_hb} returns $\max(C_{out}, L_{\mathrm{last}})$.
Evaluating the major and minor candidates here would sum a leg sequence that
does not exist and would overestimate the peak by up to $128$ scalars on the
configurations of \cref{app:validation}, which is how the corruption simulator
isolated it.

\paragraph{Scope of the optimality claim.}
\cref{thm:generalized_hb} establishes the exact peak debt of the generalized
\emph{herringbone traversal}. Whether herringbone remains memory-optimal
\emph{among all traversals} under arbitrary parameters is a separate question,
and the answer is no. Gural~\&~Murmann's optimality
argument~\cite[Lemma~3.1]{pmlr-v97-gural19a} rests on a Chebyshev-ball
isoperimetric bound that assumes an isotropic square kernel and does not extend
to anisotropic receptive fields. More directly, when $s_h > k_h^{\mathrm{eff}}$
or the dilation is large enough, some input pixels are stale on arrival: no
output ever depends on them. A traversal that discards those before storing the
input and then runs herringbone uses strictly less memory, so herringbone
without such a pruning step cannot be optimal in that regime. We do not pursue
the pruned variant here. \cref{thm:generalized_hb} is used only as a baseline
against which RiP is compared, and the figures in
\cref{tab:sota_layer_results} are herringbone peak debts rather than
traversal-optimal minima.

\section{Dense Input Layout: Padding Not Materialized}
\label{app:dense}

\cref{eq:layout} stores the padded input in full. The
$(H_pW_p - H_{in}W_{in})C_{in}$ border cells all hold the same pad value and
carry no information, yet on a fixed-workspace device, they occupy memory addresses that nothing else may use. This appendix re-derives the minimum gap when they
are not stored at all, and settles the comparison between the two layouts,
which is less obvious than it looks.

\subsection{Layout and the Extended Next Live Input}
\label{app:dense_layout}

Under the dense layout the workspace holds only the $H_{in}W_{in}C_{in}$ real
scalars, packed row-major:

\begin{equation}
\underbrace{[\,0,\ldots,0\,]}_{\Delta\text{ scalars}}
\;\Big|\;
\underbrace{X[0,0,0],\ldots,X[H_{in}{-}1,W_{in}{-}1,C_{in}{-}1]}_{\text{real input, densely packed}},
\label{eq:dense_layout}
\end{equation}

with real pixel $(r,c)$ at $\Delta + (rW_{in}+c)C_{in}$. A tap landing outside
$[0,H_{in})\times[0,W_{in})$ reads an implicit pad value and touches no
workspace address.

\cref{eq:nli} then no longer suffices. It located the first still-needed
address by advancing one column, or wrapping one row, in the \emph{padded}
index space, where every position has an address. Under the dense layout the
next tap position may itself fall in the unstored region, where there is
nothing to report, so the dense next live input must skip forward to the first
\emph{real} pixel $(\bar r, \bar c)$ still required by an unprocessed output:

\begin{equation}
\mathrm{NLI}(r,c) =
\begin{cases}
\begin{aligned}[t]
&\Delta + (\bar{r} W_{in} + \bar{c})\,C_{in} \\
&\quad \text{if } (r,c) \neq (H_{out}{-}1,\, W_{out}{-}1),
\end{aligned} \\[8pt]
\begin{aligned}[t]
& +\infty \quad \text{otherwise (last pixel).}
\end{aligned}
\end{cases}
\label{eq:nli_dense}
\end{equation}

The skip is what costs the closed form its $O(1)$ evaluation. In the padded
layout the next live input advances by a fixed stride at every step, so the
debt is affine everywhere except at a row end. Here it advances by a fixed
stride only while no tap changes side of the padding boundary, and the debt
kinks wherever one does.

\subsection{The Bound}
\label{app:dense_bound}

\begin{lemma}[Dense affine debt]
\label{lem:dense_affine_debt}
Under the dense layout the total debt is
\begin{equation}
\label{eq:debt_dense}
\mathcal{D}(r,c) = \big(r W_{out} + (c{+}1)\big)C_{out}
                   - (\bar{r} W_{in} + \bar{c})C_{in} + \sigma,
\end{equation}
affine in $(r,c)$ between the indices at which a tap crosses the padding
boundary. On the column axis tap $i \in [0,k_w)$ enters the real region at
$c = \lceil (p_w - i d_w)/s_w \rceil$ and leaves it at
$c = \lceil (p_w - i d_w + W_{in} - 1)/s_w \rceil$; the row axis is analogous
with $j \in [0,k_h)$, $p_h$, $d_h$, $s_h$ and $H_{in}$.
\end{lemma}

\begin{proof}
Only the subtracted term of \cref{eq:delta_rc} changes, the input now being
indexed densely. Between two consecutive crossings, no tap changes side of the
boundary, so $(\bar{r}, \bar{c})$ advances by $s_w$ columns per output column and
$s_h$ rows per output row exactly as in \cref{lem:piecewise_affine_debt}, and
\cref{eq:debt_dense} is affine there with the same coefficients. Tap $i$
crosses only where $cs_w + i d_w - p_w$ reaches $0$, entering the real region,
or $W_{in}$, leaving it; solving each for $c$ gives the two stated
indices.\qed
\end{proof}

\begin{theorem}[Minimum gap, dense layout]
\label{thm:gap_dense}
Assume $p_h < k_h^{\mathrm{eff}}$ and $p_w < k_w^{\mathrm{eff}}$, so that no
output pixel reads padding alone. Extend $R_s$ and $C_s$ of \cref{thm:gap} with
every crossing index of \cref{lem:dense_affine_debt} and the index below it,
clipped to the grid. Then \cref{eq:delta_star} with $H_pW_p$ replaced by
$H_{in}W_{in}$ gives a safe gap $\Delta_{\mathrm{d}}^{*}$, and the workspace is
$\mathcal{W}_{\mathrm{d}}^{*} = \Delta_{\mathrm{d}}^{*} + H_{in}W_{in}C_{in}$.
The extended sets have at most $4k_h+3$ and $4k_w+3$ members, so the evaluation
is $O(k_hk_w)$.
\end{theorem}

\begin{proof}
The crossings partition the output grid into cells on which $\mathcal{D}$ is
affine by \cref{lem:dense_affine_debt}, so its maximum on each cell is attained
at a breakpoint. Taking every crossing index together with its predecessor covers
both sides of each partition boundary, and the three extremes \cref{eq:breakpoints} on each axis
cover the outermost cells; each of the $k_w$ column taps contributes two
crossings and each crossing two indices, giving the stated cardinality. Safety
and the workspace identity then follow exactly as in \cref{thm:gap}, with
$H_pW_pC_{in}$ replaced by $H_{in}W_{in}C_{in}$. The precondition ($p_h < k_h^{\mathrm{eff}}$ and $p_w < k_w^{\mathrm{eff}}$) is what
guarantees $(\bar r, \bar c)$ exists for every unprocessed output; with padding
at least as wide as the dilated kernel some output would read no real input at
all, and the dense layout could not represent it.\qed
\end{proof}

\subsection{Why the Comparison Must Be Made in Workspace}
\label{app:dense_vs_padded}

Dropping the padding \emph{raises} the gap. This is the right way round: the
padded layout pre-allocates $H_pW_pC_{in}$ scalars, of which the border is
slack the output may consume for free, and the dense layout does not have that
slack to spend. Reported as a gap, the dense layout therefore looks like a
regression.

What has to fit in SRAM is the workspace, and there the sign reverses:

\begin{equation}
\mathcal{W}^{*} - \mathcal{W}_{\mathrm{d}}^{*}
= \underbrace{(H_pW_p - H_{in}W_{in})C_{in}}_{\text{border never allocated}}
- \underbrace{\big(\Delta_{\mathrm{d}}^{*} - \Delta^{*}\big)}_{\text{extra slack carried}} .
\label{eq:dense_saving}
\end{equation}

On every configuration in the sweep of \cref{app:validation} the first term
exceeds the second, so the difference is non-negative, and strictly positive
whenever the padding is non-zero. The saving is largest exactly where padding
is largest relative to the feature map: \cref{tab:sota_layer_results} reports a
mean of $2.1\%$ across the 84 layers but $17.4\%$ on Inception-v3, whose
$7{\times}1$ and $1{\times}7$ branches pad one axis heavily while the other
carries none.

Two layouts that disagree in gap and agree in workspace is not a presentational
detail. It is the reason every comparison in this paper is stated in total
scalars: a per-layer bound is only useful to the extent that it predicts what
the allocator must reserve, and the gap alone does not.

\subsection{Tightness When Input Is Stale From the Outset}
\label{app:dense_gcd}

\cref{thm:gap_dense} is safe unconditionally but not always tight. The gap
opens when a real input scalar is never read by any output, so that it is dead
before the layer begins. Input column $x$ is read only if
$x = cs_w + jd_w - p_w$ for some output $c$ and tap $j$, and two mechanisms
leave that unsatisfiable. If $s_w > k_w^{\mathrm{eff}}$ the columns between
$cs_w + k_w^{\mathrm{eff}}$ and $(c{+}1)s_w$ fall past one receptive field and
before the next. If instead $d_w$ is large relative to $s_w$, the columns
strictly between consecutive taps, $cs_w + jd_w$ and $cs_w + (j{+}1)d_w$, are
skipped within every receptive field that could cover them; the row axis is analogous.

Dead input is necessary for looseness but not sufficient, and two conditions
must hold together. First, the dead addresses must include the \emph{first}
one: because RiP writes to fixed physical addresses rather than reusing a free
list, what matters is not how many are dead but which, and only a dead leading
run lets the write frontier start earlier than the bound permits. Second,
$\sigma$ must be the binding term, $\Delta_{\mathrm{d}}^{*} = \sigma$. The
headroom exists because output pixel $(0,0)$ does not overwrite its own taps
before reading them, and those taps begin at the first live address; a dead
leading run of $L$ pixels covers $L\,C_{in}$ scalars of it in advance. But
\cref{eq:sigma} is a maximum, so when an interior breakpoint
dominates, discounting the headroom leaves that maximum unchanged and the bound
stays exact. Empirically, the overshoot is at most $\min(\sigma,\, L\,C_{in})$ and never
exceeded $\sigma$ on any configuration we test.

The simplest configuration that is loose has $C_{in}=C_{out}=2$,
$k = 1\times2$, $H_{in}=W_{in}=1$, unit stride, $(d_h,d_w) = (1,2)$ and
$(p_h,p_w) = (0,1)$, so that $k_w^{\mathrm{eff}} = 3$. The two taps land at
$-1$, which is padding, and at $+1$, which is past the end of the input, so
the single real column is never read. The true minimum gap is $0$: the output
may be written straight over an input that nothing will ever load.
\cref{thm:gap_dense} returns $2$, the full production-order headroom
$\sigma = C_{out}$, because it is unaware of this stale dependency of the first pixel.

\section{Validation Protocol}
\label{app:validation}

Every closed form in this paper is checked against a simulator that does not
know the formula. This appendix describes the simulator, how the $10{,}000$
configurations of \cref{sec:synthetic} are composed, and what the check does
and does not establish.

\subsection{The Corruption Simulator}
\label{app:simulator}

The simulator executes a traversal literally in an allocation-tracked
workspace. It holds a scalar-level map of the buffer, marks each input scalar
live until the last output pixel that depends on it has been computed, and flags a
corruption the moment a write lands on an input tap still marked live. Given a
layer configuration it binary-searches the gap: the smallest $\Delta$ that
completes the traversal with no flag is the ground truth reported as
\textbf{Truth} in \cref{tab:limitations}.

Two properties make this a check rather than a restatement. The simulator never
evaluates any closed form, so agreement is independent evidence. And it returns
the minimum rather than a safe value, so it catches over-allocation as well as
under-allocation: this is what exposed Limitation~2, where the published
form exceeds truth by up to $2{,}432\times$, and the degenerate leg sequence of
\cref{app:generalized_hb_degenerate}.

Each configuration therefore receives one of three verdicts per closed form:
\textsc{unsafe} if the formula falls below truth, which is the only outcome
that admits a corruption; \textsc{tight} if it equals truth; and \textsc{loose}
if it exceeds truth by any number of scalars. No sweep in this paper has
produced an \textsc{unsafe} verdict.

\subsection{Sweep Composition}
\label{app:sweeps}

We draw $10{,}000$ layer configurations from the twenty generators of
\cref{tab:sweeps}, with input spatial extent up to $128$ and a fixed seed. The
generator weights sum to $100$ and so read directly as percentages of the
sweep. They fall in three tiers. The first, $35\%$, is what real models are
built from: plain and same-padded convolutions, strided downsampling, inverted
bottlenecks, and projection layers. The second, $45\%$, is small-grid and
near-degenerate geometry, where every defect we have found lives: the
collapsed minor axis, grids of six pixels or fewer, thin axes, and the
non-integer $x^*$ that drives Limitation~1. The third, $20\%$, covers
correctness corners that are rare in shipping models but must be covered
because generality is the claim: rectangular kernels, dilation, and asymmetric
stride, dilation and padding.

Each layer configuration is run under all five execution configurations that the
schemes admit: RiP padded and RiP dense, each with the receptive field gathered
first or read from the workspace, and generalized herringbone reading from the
workspace. Herringbone under gather-first is not implemented.

\begin{table}[t]
\centering
\caption{The twenty generators, with their share of the sweep. Tier~1 is the
geometry shipping models are built from, tier~2, the small-grid and degenerate
regimes where every defect found so far has lived, and tier~3, the correctness
corners that generality requires but real models rarely reach.}
\label{tab:sweeps}
\small
\begin{tabular}{@{}llr@{}}
\toprule
Tier & Generator & Share \\
\midrule
1 & plain convolution, $k \in [1,7]$      & $11\%$ \\
1 & same padding, $p = \lfloor k/2 \rfloor$ & $8\%$ \\
1 & strided downsampling                   & $6\%$ \\
1 & inverted-bottleneck expansion          & $5\%$ \\
1 & projection, $C_{out} \le C_{in}$       & $5\%$ \\
\midrule
2 & collapsed minor axis, $n \in \{1,2,3\}$ & $9\%$ \\
2 & small herringbone grid, $n \le 6$      & $7\%$ \\
2 & small output grid                      & $6\%$ \\
2 & thin axis, one output axis of $1$      & $6\%$ \\
2 & restricted regime, integer $x^*$       & $6\%$ \\
2 & random symmetric axes                  & $6\%$ \\
2 & non-integer $x^*$                      & $5\%$ \\
\midrule
3 & non-monotone tap order                 & $4\%$ \\
3 & rectangular kernel                     & $3\%$ \\
3 & dilation                               & $3\%$ \\
3 & asymmetric stride                      & $2\%$ \\
3 & asymmetric dilation                    & $2\%$ \\
3 & asymmetric padding                     & $2\%$ \\
3 & fully asymmetric                       & $2\%$ \\
3 & maximal padding                        & $2\%$ \\
\bottomrule
\end{tabular}
\end{table}

\subsection{What the Sweep Does Not Establish}
\label{app:sweep_scope}

The sweep establishes that the gaps are safe and, where reported, tight. It
does not establish that a layer executed in the RiP layout computes the same
tensor as one executed conventionally, which is a separate claim and the one a
practitioner cares about. Liveness bookkeeping and arithmetic are different
questions: a gap can be safe while the kernel is wrong, and a kernel can be
right while the gap is too small. The sweep only answers the former question, a correctly implemented kernel answers the latter.

The end-to-end claim is discharged on device instead, and more strongly than a synthetic
check would. Every deployed model in \cref{tab:deployment} carries an on-device
checksum of its output logits, computed after the full network has run under
RiP and compared against the same checksum under the baseline schedule. All eleven
agree bit for bit. Since RiP changes only where activations are placed and not how they are computed (that is the kernel implementation's business), the agreement over eleven networks is the observed and expected behavior.

\section{Per-Layer Results}
\label{app:layer_results}

\cref{sec:real_layers} reports the aggregate comparison across 84 convolutional
layers. \cref{tab:sota_layer_results} breaks it down by architecture family, so
that the aggregate can be checked against the families it is drawn from rather
than taken on trust. Two structural facts are easier to see here than in the
aggregate.

First, the layers on which RiP costs anything against herringbone are
concentrated, not spread. Eleven of the seventeen families contain no layer
above a $1.01$ workspace ratio, and the three layers exceeding $5\%$ sit in two
families: pruned VGG-16, whose channel-doubling layers run on $16{\times}16$
and $8{\times}8$ grids, and Inception-v3, whose $7{\times}1$ branch is the
single worst case at $1.170$. Where the row-major frontier loses, it loses for
a reason visible in the layer's shape.

Second, the \textbf{Dense} column is zero on exactly the 46 layers that use no
padding, and non-zero on all 38 that do. The saving from not materializing
padding is therefore not a property of the method but of the layer: it is
bounded by the padding fraction of the input, which is why it reaches $17.4\%$
on Inception-v3, whose $7{\times}1$ and $1{\times}7$ branches pad one axis and
not the other, and $25.3\%$ on the single worst layer. Across the whole set
the dense gap is never smaller than the padded gap and the dense workspace never larger, strictly so on the 38 layers that pad; the sign inversion of \cref{eq:dense_saving} holds on real layers as well as on the sweep of \cref{app:validation}.

\begin{table*}[t]
\centering
\caption{RiP against the generalized herringbone bound on 84 convolutional
layers drawn from 25 architectures. $n$ is layers sampled; \textbf{Gen.} counts
those whose parameters fall outside the restrictions
of~\cite{pmlr-v97-gural19a} and so are reachable only through
\cref{thm:generalized_hb}. \textbf{Exact} and \textbf{${\le}5\%$} count layers
where RiP's workspace equals, or is within 5\% of, herringbone's. \textbf{Max}
is the worst RiP-to-herringbone workspace ratio in the family. \textbf{Save} is
the mean workspace reduction against dual buffering, and \textbf{Dense} the
mean further reduction from not materializing padding (\cref{app:dense}).
\textbf{Timed} counts the layers executed on the Pico~2 and \textbf{Cost} is
their median runtime ratio against the dual-buffer kernel, right-shift
included; a dash marks a family with no timed layer. Every layer passes the
corruption simulator under both layouts, so no correctness column is shown.}
\label{tab:sota_layer_results}
% \scriptsize
% \setlength{\tabcolsep}{3pt}
\begin{tabular}{@{}lrrrrrrrrr@{}}
\toprule
Family & $n$ & Gen. & Exact & ${\le}5\%$ & Max & Save & Dense & Timed & Cost \\
\midrule
ResNet-50            & 11 &  4 & 10 & 11 & $1.006$ & $25.5\%$ & $2.3\%$  &  6 & $1.000$ \\
ResNet-56 (pruned)   &  5 &  5 &  4 &  5 & $1.011$ & $26.9\%$ & $11.4\%$ &  5 & $1.001$ \\
VGG-16               &  4 &  4 &  0 &  4 & $1.033$ & $25.0\%$ & $0.9\%$  & -- & -- \\
VGG-16 (pruned)      &  3 &  3 &  0 &  1 & $1.086$ & $17.9\%$ & $4.2\%$  &  3 & $1.000$ \\
MobileNetV1          &  6 &  1 &  5 &  6 & $1.005$ & $32.1\%$ & $0.0\%$  &  3 & $1.000$ \\
MobileNetV2          &  5 &  1 &  4 &  5 & $1.005$ & $16.8\%$ & $0.0\%$  &  2 & $1.000$ \\
MCUNet-in1--in4      & 16 &  4 & 12 & 16 & $1.022$ & $24.2\%$ & $0.1\%$  & 16 & $1.002$ \\
EfficientNet-B0      &  4 &  1 &  3 &  4 & $1.005$ & $17.4\%$ & $0.0\%$  &  1 & $1.000$ \\
ConvNeXt-T           &  2 &  2 &  1 &  2 & $1.007$ & $33.1\%$ & $0.0\%$  &  2 & $1.001$ \\
DenseNet-121         &  2 &  1 &  1 &  2 & $1.006$ & $24.3\%$ & $0.1\%$  & -- & -- \\
ShuffleNetV2         &  1 &  1 &  0 &  1 & $1.007$ & $32.9\%$ & $0.2\%$  &  1 & $0.997$ \\
SqueezeNet           &  2 &  2 &  0 &  2 & $1.009$ & $15.3\%$ & $0.3\%$  &  1 & $1.000$ \\
Xception             &  1 &  1 &  0 &  1 & $1.002$ & $27.2\%$ & $0.0\%$  & -- & -- \\
YOLO v5--v26         & 15 &  7 & 13 & 15 & $1.003$ & $29.4\%$ & $1.2\%$  &  5 & $1.000$ \\
RT-DETR-l            &  3 &  2 &  3 &  3 & $1.000$ & $18.1\%$ & $0.4\%$  & -- & -- \\
Inception-v3         &  2 &  2 &  1 &  1 & $1.170$ & $29.0\%$ & $17.4\%$ &  2 & $1.000$ \\
DeepSpeech2          &  2 &  2 &  1 &  2 & $1.021$ &  $9.9\%$ & $10.9\%$ &  1 & $1.008$ \\
\midrule
\textbf{All}         & 84 & 43 & 58 & 81 & $1.170$ & $24.8\%$ & $2.1\%$  & 48 & $1.000$ \\
\bottomrule
\end{tabular}
\end{table*}

The counts are properties of the layer set rather than frequencies across whole
networks: \cref{sec:real_layers} selects the layers of largest combined input
and output footprint, which biases toward early high-resolution layers and
toward the channel-expanding pointwise convolutions of inverted-bottleneck
designs.
\section{Graph-Level Simulation and MILP Cost}
\label{app:milp_scale}

The eleven deployed models of \cref{tab:deployment} are small enough that
solver cost is not in question. To see how the MILP behaves on larger model
graphs, we stripped TinyEngine's kernel generators from its TFLite parser,
keeping the graph API and layer construction, branch connections, and tensor
liveness, and ran the allocator alone across 33 models from
\texttt{keras.applications}: DenseNet (121, 169, 201), EfficientNet (B0--B7),
EfficientNetV2 (B0--B3, S, M, L), ResNet and ResNetV2 (50, 101, 152),
MobileNet, MobileNetV2, MobileNetV3 (Small, Large), VGG (16, 19), InceptionV3,
InceptionResNetV2, and Xception. These are not deployed; the point is to study the cost of the allocator and not execution on the boards. Every layer runs under
the dense layout of \cref{app:dense} with the receptive field gathered first, so padding is never materialized and $\sigma = 0$. The allocator was run on an 8-core (16-thread) Intel processor at 1.80\,GHz base (4.50\,GHz turbo) with 64\,GB of memory, using SCIP through the Pyomo interface. Each model is solved single-threaded, with four models measured concurrently. We report the median of 10 runs per model and per formulation.

\paragraph{Scaling.} The program size is quadratic in the graph by
construction: the pairwise no-overlap constraints put three overlap binaries
on each of the $n(n{-}1)/2$ tensor pairs, so the dual-buffer formulation has
exactly $3n(n{-}1)/2$ binaries on all 33 models, and RiP adds $2m$ for the $m$
in-place candidates. Cost per binary, however, is flat. Regressing on the
number of binaries over a $46\times$ range in graph size ($n = 22$ to
$1{,}006$) gives an exponent of $1.05$ for model construction ($R^2 = 0.99$)
and $0.99$ for the branch-and-bound search ($R^2 = 0.97$): both stages are
linear in the size of the program they are given. Against tensor count the
exponents are $2.10$ and $1.98$, which is that linearity composed with the
quadratic encoding rather than a solver that degrades with scale.

\paragraph{Cost.} The largest instance, EfficientNetV2-L, has 1{,}006 tensors,
1.52M binaries and 2.03M constraints, and takes 132\,s. No model in the set
exceeds 150\,s and the median is 13.0\,s. The branch-and-bound search dominates
model construction on 30 of the 33 (96.3\,s against 34.8\,s on
EfficientNetV2-L), so the cost is the packing problem rather than the Python
that builds it. The problem grows with the graph and not with the tensors:
VGG-16 has the $9^{\mathrm{th}}$ largest activations in the set and the
smallest program (729 binaries, 0.16\,s), while EfficientNetV2-L, only
$5^{\mathrm{th}}$ by activations, has the most nodes and the largest. Against
the dual-buffer baseline the RiP formulation compiles between $0.99\times$ and
$2.36\times$, median $1.53\times$, which is the price of the extra binaries
$d_i$ and $r_i$. All of it is offline and paid once per model.

\paragraph{The baseline is not greedy.} The comparison here is against the same
MILP restricted to dual-buffer execution, so both sides solve an NP-hard
packing problem and the baseline achieves the optimal packing of never-aliased
tensors. All operations are dual buffered and a used memory space is made available only once the tensor's lifetime ends. The reductions are therefore
conservative: RiP cuts peak activation memory on all 33, by 20.0\% to 49.8\%
(mean 32.6\%), against an allocator that is already optimal within its
execution model.

\paragraph{In-place is not forced.} The objective minimizes peak memory plus
the right-shift penalty of \cref{sec:execution}. Nothing rewards choosing an op
as in-place for its own sake; it is chosen only where it reduces the model's
peak memory by more than the memmove costs. The share of tensors placed
in-place therefore follows the topology, ranging from 4.5\% on
InceptionResNetV2 to 83.7\% on Xception.

The penalty is what keeps runtime moves out: across the 33, the solver
schedules a right shift on four models only (VGG-16, VGG-19, Inception-v3, and InceptionResNetV2), one tensor each. These are also the only four whose peak ratio is not an exact fraction of the baseline, and in each case the shifted tensor is the peak-setting one: the move is what buys the aliasing at the peak, so both facts have the same cause. All four peak on a same-padded layer, where the write frontier must clear a full row before any input goes stale and the dense layout has no border slack to give it. Which term binds depends on the channel ratio. On VGG-16's \texttt{conv1\_2} ($64 \to 64$, $k = 3$, $224\times224$, same padding) both coefficients vanish under the dense layout, $\alpha = C_{out} - s_wC_{in} = 0$ and $\beta = C_{out}W_{out} - s_hW_{in}C_{in} = 0$: advancing a row frees exactly what it writes, the debt is flat across the interior, and the maximum sits at the early border kink, giving $W^{*} = N_{out} + (p_hW_{out} + p_w)C_{out} = 3{,}211{,}264 + 14{,}400 = 3{,}225{,}664$~B. The Inception pair peaks on the channel-doubling stem convolution ($32 \to 64$, $k = 3$, $147\times147$, same padding), where $\alpha = 32$ and $\beta = 4{,}704$ are positive: the debt grows with each output row and the maximum sits at the last output, giving $W^{*} = N_{out} + \beta = 1{,}382{,}976 + 4{,}704 = 1{,}387{,}680$~B. \Cref{thm:gap_dense} accounts for both. The two Inception models agree to the byte because they share the Inception-v3 stem and peak on the same layer, not by coincidence of ratios. The eleven deployed models of \cref{tab:deployment} peak instead on the pointwise convolutions of their MobileNetV2-style backbone, which carry no padding, so no move is scheduled on any of them and every ratio is exact.

\paragraph{Operator classes.} Every model layer falls in one of two classes.
Concatenation, reduction and fully-connected layers need disjoint input and 
output and are held apart by construction, so they are forced to dual
buffering regardless of the allocator mode.
Convolution, depthwise convolution, pooling, element-wise and scalar
add and multiply, upsampling, and activations are candidates for in-place
placement, and the MILP decides. Treating tensors as candidates for in-place
placement rather than being forced to be aliased matters on residual topologies,
where an activation's input may still be live for a downstream skip connection;
the liveness constraints make the program fall back to disjoint placement in
exactly those cases.

\paragraph{Upsampling.} To gain the full advantage of in-place placement we
extend the workspace formulation to the other local operators, pooling and sampling. Person-Det needs this specifically, since upsampling is the only operator here whose output is larger than its input on both spatial axes while its input is an in-place candidate. \cref{app:upsample} derives it, and shows the workspace is $\max(N_{out}, N_{in})$.

\subsection{In-Place Upsampling}
\label{app:upsample}

Nearest-neighbor upsampling by an integer factor $S \ge 1$ has $C_{out} = C_{in} \equiv C$, $H_{out} = SH_{in}$ and $W_{out} = SW_{in}$, and output pixel $(r,c)$ reads the single input pixel $(\lfloor r/S \rfloor, \lfloor c/S \rfloor)$. The forward traversal of \cref{thm:gap} does not apply: it runs forward from $(0,0)$ and output $(0,0)$ lands on input $(0,0)$, which outputs $(0,1)$ through $(S{-}1,S{-}1)$ still need.

Reversing the traversal (running the operation in reverse) fixes it, with no gap at all. We allocate one workspace
of $N_{out} = H_{out}W_{out}C$ scalars, pack the input row-major at $[0,
H_{in}W_{in}C)$ and write the output row-major from index $0$: both tensors are
\emph{left}-aligned and $\Delta = 0$. The traversal runs from $(H_{out}{-}1,
W_{out}{-}1)$ down to $(0,0)$, so the write frontier descends from the top of the workspace while the read frontier descends $S^2$ times more slowly from far below it.

Indices $(r,c)$ range over the output grid as in
\cref{sec:right_align_layout}, and in pixel units the pixel about to be written begins at
\begin{equation}
\mathrm{out}(r,c) = rW_{out} + c = rSW_{in} + c .
\label{eq:up_out}
\end{equation}
Because $\lfloor \cdot / S \rfloor$ is non-decreasing, the highest input still required is the one read by the next output in the traversal, so the next live input is
\begin{equation}
\mathrm{NLI}(r,c) =
\begin{cases}
\big\lfloor \tfrac{r}{S} \big\rfloor W_{in} + \big\lfloor \tfrac{c-1}{S}
\big\rfloor, & c > 0, \\[4pt]
\big\lfloor \tfrac{r-1}{S} \big\rfloor W_{in} + W_{in} - 1, & c = 0,\ r > 0, \\[4pt]
-\infty, & r = c = 0 \quad (\text{last pixel}).
\end{cases}
\label{eq:up_nli}
\end{equation}

\begin{theorem}[Zero-gap in-place upsampling]
\label{thm:upsample}
Under the layout above and reverse row-major traversal, no output write reaches
a live input for any $S \ge 1$, and the minimum workspace is
$\mathcal{W}^{*} = N_{out}$, the output tensor alone.
\end{theorem}

\begin{proof}
A write at $(r,c)$ is safe in reverse traversal when $\mathrm{NLI}(r,c) <
\mathrm{out}(r,c)$. Since $\lfloor r/S \rfloor \le r$ and $S \ge 1$ give
$\lfloor r/S \rfloor W_{in} \le rSW_{in}$, and $\lfloor (c{-}1)/S \rfloor \le
c-1$, the two branches of \cref{eq:up_nli} give

for $c > 0$:
\[
\mathrm{NLI}(r,c) \;\le\; rSW_{in} + c - 1 \;<\; rSW_{in} + c \;=\;
\mathrm{out}(r,c),
\]

and for $c = 0$, $r > 0$, where the bound is tightest:
\[
\mathrm{NLI}(r,0) \;\le\; rW_{in} - 1
\;\le\; rSW_{in} - 1 \;<\; \mathrm{out}(r,0).
\]

At $r = c = 0$ nothing remains unprocessed. The separation is therefore
strictly positive at every step with $\Delta = 0$, so no geometric debt is
carried and the workspace is $N_{out}=\max(N_{out}, N_{in})$. One pixel needs care: $(0,0)$ is the only index at which a write lands on the
input it reads. It costs nothing, because the value written is the value
already there (nearest-neighbor upsampling copies $\mathrm{input}(0,0)$ to
$\mathrm{output}(0,0)$, channel for channel) so the write is an identity and the tap-read order of \cref{eq:sigma} is immaterial. Hence $\sigma = 0$ under either order.
\end{proof}

Pooling needs no such extension. Max and average pooling have $C_{out} = C_{in}$ and stride $s \geq 1$, so
$\alpha = C_{out} - s_wC_{in} \leq 0$ and $\beta = C_{out}W_{out} -
s_hW_{in}C_{in} \leq 0$, which is the case $\alpha \leq 0$ and $\beta \leq 0$ of
\cref{thm:gap}: forward row-major traversal, $W^{*} = N_{in} = \max(N_{out}, N_{in})$, and the output
absorbed entirely into the space its input vacates.
Upsampling is the mirror image, which is why it is the one case that has to run
backwards.

\begin{table*}[t]
\centering
\caption{Graph-level allocation across 33 \texttt{keras.applications} models
under the dense layout with the receptive field gathered first. Ops is graph
nodes after folding and Res.\ the branch connections that keep two tensors
alive at once. Vars and Constr.\ size the MILP; Compile is the full offline
step (parse, model build, solve) under each
formulation, median of 10 runs. Dual buffer is the same program restricted to
disjoint placement, so it is the optimal packing of never-aliased tensors
rather than a greedy baseline. In-place is the share of tensors the solver
places in-place and Shift the number it separates by a physical right shift; a
dash marks a model with no move scheduled. The memory ratio is an exact fraction on 29
of the 33, and the four exceptions are exactly the four models with a non-empty
Shift column.}
\label{tab:milp_scale}
\footnotesize
\setlength{\tabcolsep}{3.5pt}
\begin{tabular}{@{}lrrrrrrrrrrrr@{}}
\toprule
& & & \multicolumn{2}{c}{MILP size} & \multicolumn{3}{c}{Compile (s)} & \multicolumn{3}{c}{Peak activation (KB)} & \multicolumn{2}{c}{Tensor placement} \\
\cmidrule(lr){4-5}\cmidrule(lr){6-8}\cmidrule(lr){9-11}\cmidrule(l){12-13}
Model & Ops & Res. & Vars & Constr. & Dual buf. & RiP & Ratio & Dual buf. & RiP & Ratio & In-place (\%) & Shift (\#) \\
\midrule
% \label{tab:milp_scale}
% \small
% \setlength{\tabcolsep}{4.5pt}
% \begin{tabular}{@{}lrrrrrrrrrr@{}}
% \toprule
% & & & \multicolumn{2}{c}{MILP size} & & \multicolumn{3}{c}{Peak activation memory} & \multicolumn{2}{c}{Tensor placement} \\
% \cmidrule(lr){4-5}\cmidrule(lr){7-9}\cmidrule(l){10-11}
% Model & Ops & Res. & Vars & Constr. & Compile (s) & Dual buffer (KB) & RiP (KB) & Ratio & In-place (\%) & Shift (\#) \\
\midrule
MobileNet            & 29 & 0 & 1{,}361 & 1{,}969 & 0.08 & 0.14 & 1.69 & 1{,}176.0 & 784.0 & 2/3 & 60.0 & -- \\
MobileNetV2          & 64 & 10 & 6{,}364 & 8{,}813 & 0.26 & 0.52 & 2.01 & 1{,}470.0 & 1{,}176.0 & 4/5 & 49.2 & -- \\
MobileNetV3-S        & 109 & 15 & 18{,}201 & 24{,}834 & 1.08 & 1.49 & 1.37 & 392.0 & 220.5 & 9/16 & 64.5 & -- \\
MobileNetV3-L        & 122 & 18 & 22{,}751 & 30{,}967 & 1.27 & 1.42 & 1.12 & 980.0 & 784.0 & 4/5 & 82.1 & -- \\
\addlinespace
VGG-16               & 21 & 0 & 729 & 1{,}077 & 0.11 & 0.16 & 1.49 & 6{,}272.0 & 3{,}150.1 & 0.502 & 54.5 & 1 \\
VGG-19               & 24 & 0 & 942 & 1{,}377 & 0.12 & 0.18 & 1.55 & 6{,}272.0 & 3{,}150.1 & 0.502 & 56.0 & 1 \\
\addlinespace
ResNet-50            & 72 & 16 & 8{,}024 & 11{,}063 & 0.59 & 0.67 & 1.15 & 2{,}352.0 & 1{,}568.0 & 2/3 & 47.9 & -- \\
ResNet-50v2          & 109 & 16 & 18{,}199 & 24{,}827 & 1.13 & 1.77 & 1.56 & 2{,}352.0 & 1{,}568.0 & 2/3 & 54.5 & -- \\
ResNet-101           & 140 & 33 & 29{,}886 & 40{,}558 & 1.84 & 2.28 & 1.24 & 2{,}352.0 & 1{,}568.0 & 2/3 & 28.4 & -- \\
ResNet-101v2         & 211 & 33 & 67{,}516 & 91{,}110 & 4.00 & 6.16 & 1.54 & 2{,}352.0 & 1{,}568.0 & 2/3 & 67.0 & -- \\
ResNet-152           & 208 & 50 & 65{,}620 & 88{,}549 & 3.95 & 5.21 & 1.32 & 2{,}352.0 & 1{,}568.0 & 2/3 & 55.5 & -- \\
ResNet-152v2         & 313 & 50 & 148{,}045 & 199{,}009 & 9.02 & 10.97 & 1.22 & 2{,}352.0 & 1{,}568.0 & 2/3 & 63.1 & -- \\
\addlinespace
DenseNet-121         & 308 & 58 & 143{,}138 & 192{,}103 & 7.67 & 14.64 & 1.91 & 2{,}058.0 & 1{,}568.0 & 16/21 & 44.7 & -- \\
DenseNet-169         & 428 & 82 & 275{,}942 & 369{,}655 & 14.49 & 29.50 & 2.04 & 2{,}058.0 & 1{,}568.0 & 16/21 & 47.6 & -- \\
DenseNet-201         & 508 & 98 & 388{,}478 & 520{,}023 & 20.64 & 28.91 & 1.40 & 2{,}058.0 & 1{,}568.0 & 16/21 & 47.5 & -- \\
\addlinespace
Inception-v3         & 124 & 35 & 23{,}394 & 31{,}683 & 1.49 & 2.46 & 1.65 & 2{,}025.8 & 1{,}355.2 & 0.669 & 25.6 & 1 \\
InceptionResNetV2    & 334 & 98 & 168{,}297 & 225{,}797 & 9.43 & 15.13 & 1.60 & 2{,}025.8 & 1{,}355.2 & 0.669 & 4.5 & 1 \\
Xception             & 103 & 12 & 16{,}270 & 22{,}227 & 1.07 & 1.05 & 0.99 & 6{,}752.8 & 4{,}051.7 & 3/5 & 83.7 & -- \\
\addlinespace
EfficientNet-B0      & 242 & 74 & 88{,}689 & 119{,}465 & 5.16 & 6.39 & 1.24 & 3{,}528.0 & 2{,}352.0 & 2/3 & 32.1 & -- \\
EfficientNet-B1      & 344 & 108 & 178{,}704 & 239{,}995 & 9.79 & 13.82 & 1.41 & 4{,}050.0 & 2{,}700.0 & 2/3 & 25.2 & -- \\
EfficientNet-B2      & 344 & 108 & 178{,}704 & 239{,}995 & 10.11 & 12.95 & 1.28 & 4{,}753.1 & 3{,}168.8 & 2/3 & 11.0 & -- \\
EfficientNet-B3      & 389 & 123 & 228{,}339 & 306{,}400 & 13.09 & 20.03 & 1.53 & 9{,}492.2 & 6{,}328.1 & 2/3 & 23.1 & -- \\
EfficientNet-B4      & 479 & 153 & 345{,}834 & 463{,}510 & 18.24 & 43.12 & 2.36 & 15{,}229.7 & 10{,}153.1 & 2/3 & 15.4 & -- \\
EfficientNet-B5      & 581 & 187 & 508{,}371 & 680{,}736 & 26.90 & 46.19 & 1.72 & 21{,}930.7 & 14{,}620.5 & 2/3 & 14.8 & -- \\
EfficientNet-B6      & 671 & 217 & 677{,}706 & 906{,}966 & 36.11 & 67.09 & 1.86 & 39{,}204.0 & 26{,}136.0 & 2/3 & 9.8 & -- \\
EfficientNet-B7      & 818 & 266 & 1{,}006{,}545 & 1{,}346{,}153 & 56.70 & 76.88 & 1.36 & 50{,}625.0 & 33{,}750.0 & 2/3 & 27.4 & -- \\
\addlinespace
EffNetV2-B0          & 269 & 86 & 109{,}479 & 147{,}317 & 6.63 & 7.73 & 1.17 & 1{,}274.0 & 882.0 & 9/13 & 8.9 & -- \\
EffNetV2-B1          & 328 & 107 & 162{,}520 & 218{,}332 & 9.73 & 16.05 & 1.65 & 1{,}462.5 & 1{,}012.5 & 9/13 & 35.6 & -- \\
EffNetV2-B2          & 343 & 112 & 177{,}670 & 238{,}607 & 10.99 & 16.93 & 1.54 & 1{,}716.4 & 1{,}188.3 & 9/13 & 29.7 & -- \\
EffNetV2-B3          & 403 & 132 & 245{,}020 & 328{,}707 & 15.54 & 18.14 & 1.17 & 2{,}856.4 & 1{,}977.5 & 9/13 & 24.0 & -- \\
EffNetV2-S           & 503 & 167 & 381{,}270 & 510{,}872 & 21.77 & 33.63 & 1.54 & 5{,}616.0 & 3{,}888.0 & 9/13 & 5.6 & -- \\
EffNetV2-M           & 726 & 242 & 793{,}151 & 1{,}061{,}161 & 42.32 & 59.06 & 1.40 & 8{,}775.0 & 6{,}075.0 & 9/13 & 23.4 & -- \\
EffNetV2-L           & 1{,}005 & 337 & 1{,}518{,}551 & 2{,}029{,}754 & 80.45 & 132.04 & 1.64 & 11{,}700.0 & 8{,}100.0 & 9/13 & 23.4 & -- \\
\bottomrule
\end{tabular}
\end{table*}

\fi

\end{document}
\end{document}